\documentclass[preprint]{IEEEtran}

\usepackage{cite}

\usepackage{xcolor}
\ifCLASSINFOpdf
  \usepackage[pdftex]{graphicx}
\else
  \usepackage[dvips]{graphicx}
\fi

\usepackage{amsmath,amssymb}
\usepackage{mathtools}   
\usepackage{amsthm}      
\usepackage{bigints}

\usepackage{algpseudocode}        
\usepackage[ruled,vlined,linesnumbered]{algorithm2e} 

\usepackage{array}

\usepackage{booktabs}   

\ifCLASSOPTIONcompsoc
  \usepackage[caption=false,font=normalsize,labelfont=sf,textfont=sf]{subfig}
\else
  \usepackage[caption=false,font=footnotesize]{subfig}
\fi

\usepackage{fixltx2e}
\usepackage{stfloats}

\ifCLASSOPTIONcaptionsoff
  \usepackage[nomarkers]{endfloat}
 \let\MYoriglatexcaption\caption
 \renewcommand{\caption}[2][\relax]{\MYoriglatexcaption[#2]{#2}}
\fi

\usepackage{url}
\usepackage{hyperref}
\hypersetup{
    colorlinks=true,
    linkcolor=blue,
    filecolor=magenta,      
    urlcolor=cyan,
    citecolor=green,
    }

\newcommand{\defeq}{\vcentcolon=}

\DeclarePairedDelimiter{\abs}{\lvert}{\rvert}  
\DeclarePairedDelimiter{\norm}{\lVert}{\rVert} 

\DeclareMathOperator{\diag}{diag}

\DeclarePairedDelimiterX{\infdivx}[2]{(}{)}{%
  #1\;\delimsize\|\;#2%
}

\makeatletter
\newcommand{\@giventhatstar}[2]{\left(#1\;\middle|\;#2\right)}
\newcommand{\@giventhatnostar}[3][]{#1(#2\;#1|\;#3#1)}
\newcommand{\giventhat}{\@ifstar\@giventhatstar\@giventhatnostar}
\makeatother

\theoremstyle{plain}              
\newtheorem{theorem}{Theorem}
\newtheorem{corollary}[theorem]{Corollary}
\newtheorem{lemma}[theorem]{Lemma}
\newtheorem{proposition}[theorem]{Proposition}

\def\lb{\label}

\def\d{\mathrm{d}}

\def\ain{I^+}

\def\sfA{\mathsf{A}}
\def\sfB{\mathsf{B}}

\def\sfT{\mathsf{T}}

\def\calB{\mathcal{B}}

\def\calF{\mathcal{F}}

\def\calI{\mathcal{I}}

\def\calN{\mathcal{N}}

\def\calP{\mathcal{P}}

\def\calX{\mathcal{X}}
\def\calY{\mathcal{Y}}

\def\bbR{\mathbb{R}}
\def\bbS{\mathbb{S}}

\def\Var{\mathrm{Var}}
\def\Cov{\mathrm{Cov}}
\def\Tr{\mathrm{Tr}}
\def\Vol{\mathrm{Vol}}

\def\beq{\begin{equation}}
\def\eeq{\end{equation}}
\def\lb{\label}

\def\bft{\mathbf t}
\def\bfx{\mathbf x}

\def\bfA{\mathbf A}

\def\bfE{\mathbf E}

\def\bfI{{\mathbf I}}

\def\bfO{\mathbf O}
\def\bfP{\mathbf P}

\def\bfU{\mathbf U}
\def\bfX{\mathbf X}
\def\bfV{\mathbf V}

\def\bfX{\mathbf X}
\def\bfY{\mathbf Y}

\def\bfmu{\boldsymbol\mu}
\def\bfomega{\boldsymbol\omega}

\def\bfLambda{\boldsymbol\Lambda}
\def\bfSigma{\boldsymbol\Sigma}

\def\1{\mathds{1}} 

\def\bee{\begin{equation}}
\def\eeq{\end{equation}}

\begin{document}

\title{PRIM-cipal components analysis}

\author{Tianhao Liu,
	Daniel Andr\'es D\'iaz--Pach\'on,~\IEEEmembership{Member,~IEEE,}
        and~J. Sunil Rao

\thanks{D.~A.~D\'iaz--Pach\'on and T.~Liu are with the Division of Biostatistics, University of Miami, Miami, FL, 33136 USA (e-mail: ddiaz3@miami.edu, txl1001@miami.edu).}
\thanks{J.~S.~Rao is with University of Minnesota, Minneapolis, MN, 55414 USA (e-mail: js-rao@umn.edu).}
}


\maketitle

\begin{abstract}
Supervised No Free Lunch Theorems (NFLTs) are well studied, yet unsupervised NFLTs remain underexplored. For elliptical distributions, we prove that there exist two equally optimal, scientifically meaningful bump-hunting strategies that are exact opposites, with no universal winner. Specifically, peeling $k$ orthogonal dimensions from $\bbR^d$ ($d \ge k$), retaining an inter-quantile region of probability $1-\alpha$ per peeled dimension, maximizes total variance and Frobenius norm when the $k$ smallest principal components (called pettiest components) are selected, and minimizes them when the selected dimensions are the $k$ leading principal components. These optima inspire PRIM-based bump-hunting algorithms either by minimizing variance or by minimizing volume, thereby motivating an NFLT. We test our results on the Fashion-MNIST database, showing that peeling the largest principal components captures multiplicity, while peeling the smallest principal components isolates popular styles.
\end{abstract}


\IEEEpeerreviewmaketitle

\section{Introduction}\label{sec1}

\IEEEPARstart{E}{ven} supervised learning is subject to the famous No Free Lunch Theorems \cite{Montanez2017b, Montanez2017a, MontanezEtAl2019}, which say that, in combinatorial optimization, there is no universal algorithm that works better than its competitors for every objective function \cite{WolpertMacReady1995, WolpertMacReady1997, Wolpert2002}. Indeed, David Wolpert has recently proven that, on average, cross-validation performs as well as anti-cross-validation (choosing among a set of candidate algorithms based on which has the worst out-of-sample behavior) for supervised learning. Still, he acknowledges that ``it is hard to imagine any scientist who would not prefer to use [cross-validation] to using anti-cross-validation'' \cite{Wolpert2021}. 

On the other hand, unsupervised learning has seldom been studied from the perspective of the NFLTs. This may be because the adjective ``unsupervised'' suggests that no human input is needed, which is misleading as many unsupervised tasks are combinatorial optimization problems that depend on the choice of the objective function. For instance, it is well known that, among the eigenvectors of the covariance matrix, Principal Components Analysis selects those with the largest variances \cite{Mardia1979}. However, mode-hunting techniques that rely on spectral manipulation aim at the opposite objective: selecting the eigenvectors of the covariance matrix with the smallest variances \cite{LiuEtAl2023, SandoHino2020}. Therefore, unlike in supervised learning, where it is difficult to identify reasons to optimize with respect to anti-cross-validation, in unsupervised learning there are strong reasons to reduce dimensionality for variance {\it minimization}. 

Accordingly, we prove here that for random vectors in the elliptical family in $\bbR^d $, a sharp optimality duality governs unsupervised mode-hunting. Specifically, Theorem \ref{TheTh} shows that simultaneously peeling $ \alpha/2$ probability mass from each tail of $k$ orthogonal coordinates, while retaining a central inter-quantile region of probability $1-\alpha$ per coordinate, maximizes the preserved total variance and Frobenius norm precisely when the peeled coordinates are the $k$ smallest principal components (the pettiest components), and minimizes these quantities when the peeled coordinates are the $k$ leading principal components. However, Proposition \ref{Th2} establishes the dual geometric fact: the volume of the retained $k$-dimensional box of probability $1-k\alpha$ is minimized when the pettiest components are peeled and maximized when the leading components are chosen. These two results together reveal a fundamental trade-off in bump-hunting between minimizing either the preserved variance or the preserved box volume, a feature that has no counterpart in supervised learning.

This duality directly inspires two fastPRIM-based mode-hunting algorithms. Algorithm \ref{FastPRIMI} (FastPRIM for minimized $k$-dimensional box volume) rotates the data into the principal-component basis and peels the tails of the $k$ pettiest components simultaneously; the resulting $k$-dimensional box $\sfT_k$ of probability $\beta$ minimizes box volume while maximizing preserved total variance on the cylinder $\bbR^{d-k} \times \sfT_k$. Algorithm \ref{FastPRIMII} (FastPRIM for minimized preserved variance) peels the tails of the $k$ leading principal components instead, producing the smallest preserved total variance on the cylinder $\sfT^k \times \bbR^{d-k}$ in the largest-volume orthogonal box $\sfT^k$ of probability $\beta$ in $k$ dimensions. Moreover, both algorithms operate in a single simultaneous-peeling step, circumventing the combinatorial cost of classical PRIM while inheriting strong theoretical guarantees for elliptical distributions.

The existence of these two diametrically opposed yet equally well-justified strategies immediately yields an unsupervised No Free Lunch result for dimension selection (Proposition 4): when admissible strategies are restricted to subsets of the eigenbasis, no single rule dominates across every objective. In contrast to the supervised case, both extremes are scientifically meaningful here: mode hunting via minimization of either variance or volume. We illustrate both algorithms on the Fashion-MNIST dataset \cite{HanRasulVollgraf2017}. After class-wise PCA (retaining components via the spectral gap), peeling along the leading principal components isolates the most common clothing styles, while peeling along the pettiest components captures stylistic multiplicity (e.g., flip-flops appear only in the pettiest subset for sandals). The visual contrast on t-SNE projections and grid-sampled images underscores how the theoretical duality identified in the theoretical results translates directly into complementary scientific insights.

\section{Theoretical justification}\lb{S:Theory}

Let $\bfX \in \bbR^d$ be a random vector with mean $\bfmu$ and full-rank covariance matrix $\bfSigma$. We say that $\bfX$ belongs to the elliptical family if its characteristic function $\psi_\bfX(\bft)$ can be written as
\begin{align}\lb{CFED}
	\psi_\bfX(\bft) = e^{i \bft\bfmu}\phi\left(\bft'\bfSigma \bft\right), 
\end{align}
where $\bft= (t_1, \ldots, t_d)$ is a nonrandom vector in $\bbR^d$ and $\phi$ is a scalar function. We will denote a member of the elliptical family as $\bfX\sim (\bfmu,\bfSigma,\phi)$ or simply as $\bfX\sim (\bfmu,\bfSigma)$ when there is no risk of confusion regarding $\phi$.  Equation \eqref{CFED} implies directly that elliptical distributions are affine transformations of spherical ones:
\begin{align}\lb{AffED}
	\bfX = \bfmu + \bfA \bfO,
\end{align}
where $\bfA$ is such that $\bfA\bfA'=\bfSigma$ and $\bfO$ is a spherical random vector (i.e., its characteristic function satisfies $\psi_\bfO(\bft) = \phi\left(\bft'\bft\right)$). Equation \eqref{CFED} also implies that $\bfX$ admits the stochastic representation
\begin{align}\lb{StochRepED}
	\bfX = L\bfA\bfU,
\end{align}
where $L$ is a nonnegative random variable and $\bfU$ is uniform on the sphere $\bbS^{d-1}$, which implies in turn that when $\bfX$ is spherical, it can be written as $\bfX= L\bfI\bfU = L\bfU$ \cite{BoenteBarreraTyler2014, FangKotzNg1990}. Now, by the spectral decomposition theorem, 
\begin{align}\lb{SDT}
	\bfSigma = \bfV\bfLambda \bfV', 
\end{align}
where $\bfV$ is a $d\times d$ matrix whose columns are the normalized eigenvectors $v_1, \ldots, v_d$ and $\bfLambda$ is a diagonal matrix with eigenvalues $\lambda_1 \ge \lambda_2 \ge \cdots \ge \lambda_d > 0$ on the diagonal. 

\begin{figure}[t]
 \centering
{\includegraphics[width=5.4cm,height=3.5cm]{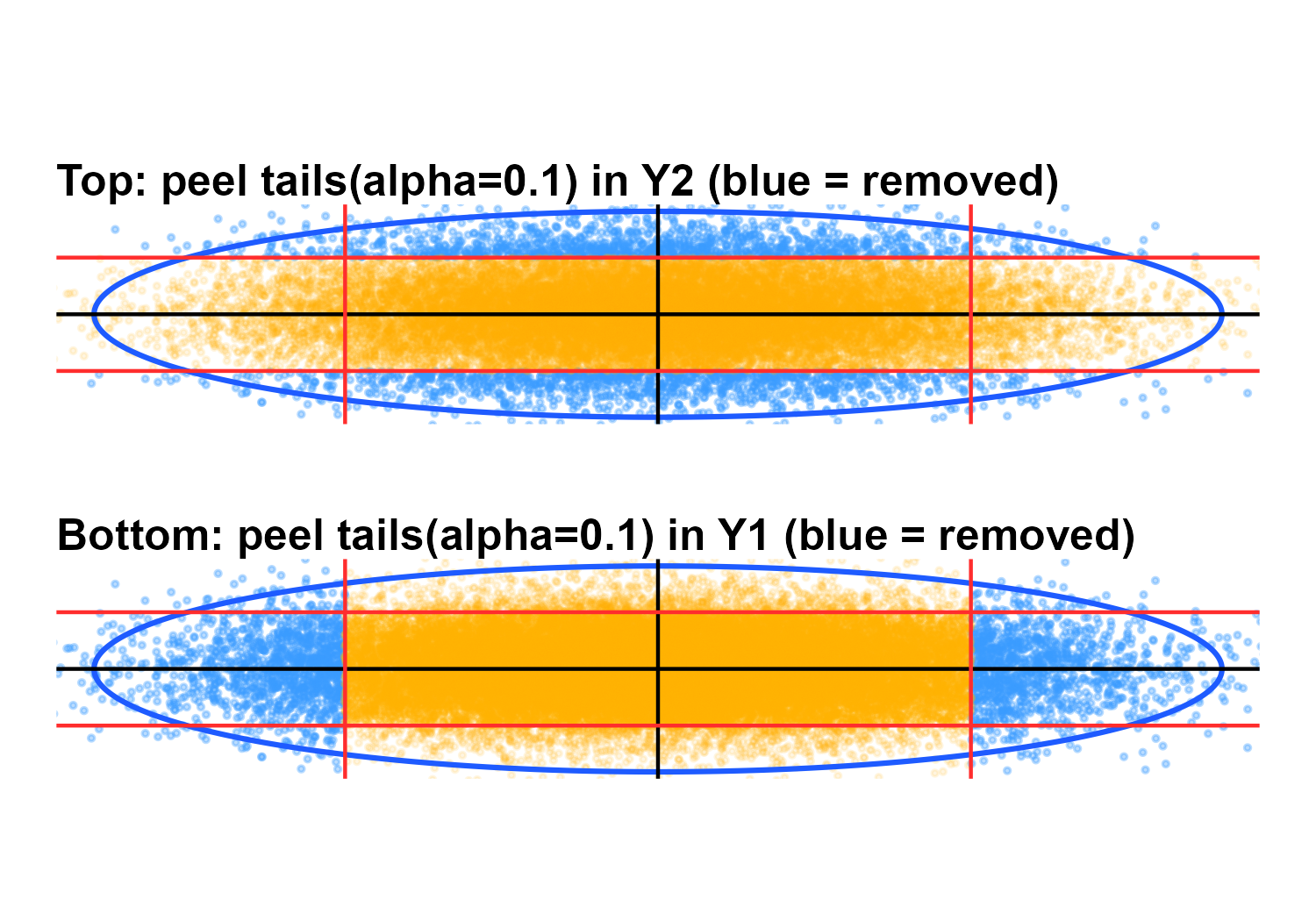}}
  \caption{Level set of a 2-dimensional elliptical distribution.}\label{Peel}
\end{figure}

Let $\calP$ be the set of all rank-$d$ orthogonal matrices, and for $\bfP \in \calP$ define the orthogonal transformation $\tilde\bfY = \bfP\bfX$. For each coordinate $\tilde Y_i$ of $\tilde \bfY$, let 
\begin{align}\lb{Ti}
	\sfT_{\tilde Y_i}^\alpha \defeq \left[Q_{\tilde Y_i}(\alpha/2), Q_{\tilde Y_i}(1-\alpha/2)\right]
\end{align}
be its central $(1-\alpha)$-inter-quantile (so that $Q_{\tilde Y_i}(\alpha)$ is its $\alpha$-quantile). For $[d] = \{1, \ldots, d\}$, let $\calI_k$ be the class of all subsets of $[d]$ with cardinality $k$. For $I_k = \{i_1, \ldots, i_k\} \in \calI_k$, define 
\begin{align}\lb{COV}
	\tilde \bfLambda_{I_k,\alpha} \defeq \Cov\giventhat*{\tilde\bfY}{\tilde Y_{i_j} \in \sfT_{\tilde Y_{i_j}}^\alpha \text{ for all } j \in [k]},
\end{align}
the preserved covariance after {\it simultaneously} truncating $\alpha/2$ from each tail of the coordinates in $I_k$, where $k\alpha \le 1$. In Theorem \ref{TheTh}, we show that, among all orthogonal rotations and all choices of $I_k$, the preserved total variance $\Tr \left(\tilde\bfLambda_{I_k,\alpha}\right)$ and the preserved Frobenius norm $\norm{\tilde \bfLambda_{I_k,\alpha}}_F$ are respectively maximized and minimized when the truncated coordinates correspond to the last and the first $k$ components of the eigenbasis $\bfY= (Y_1, \ldots, Y_d)' \defeq \bfV'\bfX$ (see Figure \ref{Peel} for intuition in 2 dimensions). The nontriviality of these optimization results is highlighted by two additional facts, summarized in Proposition \ref{Counter}. First, the preserved generalized variance $\abs{\tilde\bfLambda_{i,\alpha}}$ is constant over all subsets $I_k$, so the determinant alone cannot distinguish which coordinates are truncated. Second, the preserved operator norm $\norm{\tilde \bfLambda_{I_k,\alpha}}_{\rm op}$ cannot be maximized, and it is only minimized if the first principal component belongs to $I_k$, but it does not depend on the particular choice of the remaining coordinates in $I_k$.

\subsection{Optimizing the preserved variance}

\begin{theorem}\lb{TheTh} 
	Let $\bfX$ belong to the elliptical family. Then, the preserved total variance $\Tr \left(\tilde\bfLambda_{I_k,\alpha}\right)$ and preserved Frobenius norm $\norm{\tilde \bfLambda_{I_k,\alpha}}_F$ are maximized when we simultaneously truncate the last $k$ principal components, and minimized when we simultaneously truncate the first $k$ principal components.
\end{theorem}

\begin{proposition}\lb{Counter}
	Under the same conditions of Theorem \ref{TheTh},
	\begin{enumerate}
		\item The preserved generalized variance $\abs{\tilde\bfLambda_{I_k,\alpha}}$ is constant in all rotations, for fixed $k$.
		\item The preserved operator norm $\norm{\tilde \bfLambda_{I_k,\alpha}}_{\rm op}$ is minimized if $1 \in I_k$.
	\end{enumerate}
\end{proposition}

Theorem \ref{TheTh} is proven in Section \ref{Proofs}. The proof is greatly simplified when the components are assumed to be independent. Therefore, in Section \ref{Proofs} we first prove Theorem \ref{Normality}, a version of the result for a multivariate normal $\bfX$, illustrating where the proof fails when independence is not assumed. Interestingly, in Section \ref{Proofs}, we also prove Theorem \ref{Th1}, a preliminary version of Theorem \ref{TheTh} for a single dimension, whose proof was primarily suggested by Grok and validated by the authors. 

Proposition \ref{Th2} below proves that the volume of the box 
\begin{align}\lb{box}
	\sfT_{I_k, \tilde Y} \defeq \sfT_{\tilde Y_{i_1}}^{\alpha} \times \cdots \times \sfT_{\tilde Y_{i_k}}^{\alpha}
\end{align} 
is maximized when the last $k$ principal components are selected, and minimized when the first $k$ principal components are selected.

\begin{proposition}\lb{Th2}
	Let $\bfX$ belong to the elliptical family. Then, under the conditions of Theorem \ref{TheTh}, the volume of the box $\sfT_{I_k, \tilde Y}$ (with $k\alpha \le 1$) is maximized for $\sfT_{[k], Y}$ and minimized for $\sfT_{[d] \setminus [d-k], Y}$.
\end{proposition}

\section{Algorithms}\lb{S:Algos}

The Patient Rule Induction Method (PRIM) is a widely used bump-hunting algorithm introduced by Friedman and Fisher that works as follows \cite{FriedmanFisher1999}: Let $\bfX = (X_1, \ldots, X_d) \in \calX_1 \subset \bbR^d$ be a $d$-dimensional random vector with density $f_\bfX$, whose distribution is not necessarily elliptical. For a real-valued function $m:\calX_1 \to \bbR$, define
\begin{align*}
  \bar m_\sfA &= \frac{\int_\sfA m (\bfx) f_\bfX(\bfx) \d \bfx}{\int_\sfA f_\bfX(\bfx) \d \bfx},
\end{align*}
the conditional expectation of $m$ given $\bfX \in \sfA$, where $\sfA$ is a Borel set. The method aims to find a box $\sfB \subset \calX$ for which $\bar m_\sfB$ is maximized under the constraint that $\sfB$ has probability $\beta$, where $\beta$ is a tuning meta-parameter for the smallest permitted box measure. To do so, the algorithm is divided into stages:

\begin{algorithm} [t]
        \caption{FastPRIM for minimizing volume }\lb{FastPRIMI}
        \DontPrintSemicolon
	\KwIn{Elliptical random vector $\bfX \in \bbR^d$, peeling parameter $\beta \in (0,1)$.}
	\KwOut{$\giventhat*{\bfY}{\bigcap_{i=d-k+1}^d \left\{ Y_i \in \sfT_{Y_i}^\alpha \right\}}$.}
        Compute the orthogonal transformation $\bfY = \bfV' \bfX$, sorting its components by decreasing variance.\;
        Choose $k$ (say, by maximization of the spectral gap).\
        Peel simultaneously a probability $\alpha/2 = (1-\beta)/(2k)$ of each of the tails of the $k$ pettiest components.\;
\end{algorithm}

\begin{enumerate}
	\item {\bf Peeling.} Take the collection $\calB_\alpha$ of $2d$ extremal sub-boxes of probability $\alpha$ (a second tuning meta-parameter) of the form
\begin{align*}
 b_{j_-} &= \left\{ x \mid x_j < Q_{X_j} (\alpha) \right\},\\
  b_{j_+} &= \left\{ x \mid x_j > Q_{X_j} (1 - \alpha) \right\},
\end{align*}
and peels the box 
\begin{align}\lb{yi}
   b^{\ast} = \arg \max_{b \in \calB_\alpha} \bar m_{\calX_1 \setminus b}.
\end{align}
that maximizes the preserved conditional expectation among all boxes $b \in \calB_\alpha$. Update the new box to $\sfB = \calX_1 \setminus b^{\ast}$ and run the same procedure on the new $\sfB$. The process is iterated $L$ times until the final box $\sfB_1$ has probability $\beta = (1-\alpha)^L$. Update the space to $\calX_2 = \calX_1 \setminus \sfB_1$.
	\item {\bf Covering.} For $i$ from 1 to $t$, the space $\calX_i$ is peeled, obtaining a new peeled box $\sfB_i \in \calX_{i-1} \setminus \sf \bigcup_{j=1}^{i-1} \sfB_j$. The final box selected is $\boldsymbol{\sfB} \defeq \bigcup_{j=1}^t \sfB_j$.
\end{enumerate}

To correct for the greediness effect of peeling, the original PRIM algorithm includes an intermediate stage called {\it pasting}, which we ignore here. However, PRIM has multiple shortcomings. For instance, Polonik and Wang demonstrated that (1) even in dimension 2, PRIM was unable to distinguish between two modes in a mixture of normals \cite{PolonikWang2010}; (2) PRIM is painfully slow even in moderate dimensions, as it requires evaluating and comparing $\bar m_{\calX \setminus b}$ for each of the $2d$ boxes $b \in \calB_\alpha$ at each of the $Lt$ iterations of the peeling and covering stages combined; and (3) PRIM is not well-behaved in situations of near collinearity. For these reasons, Rao and Dazard introduced Local Sparse Bump Hunting (LSBH) \cite{DazardRao2010}, which under minimal assumptions, proceeds as follows: first, it makes a partition of the space via Classification and Regression Trees into, say, $\ell$ subspaces \cite{BreimanEtAl1984}; second, each subspace obtained from the partition is dimensionally reduced via sparse PCA \cite{ZouHastieTibshirani2006}; and third, PRIM is performed in each dimensionally reduced subspace. Moreover, in \cite{DazardRao2010}, the authors empirically showed that LSBH outperformed PRIM. LSBH was applied to medicine, revealing molecular heterogeneity of colon tumors  \cite{DazardRaoMarkowitz2012}. However, LSBH also has at least a couple of shortcomings: First, although faster than PRIM, it is still slow, as it must evaluate and compare $k_1 + \cdots + k_\ell$ boxes in each of the $L$ peeling iterations, where $k_i$ is the number of dimensions of partition $i$ after applying PCA, with $1 \le i \le \ell$. Second, since Dazard and Rao's analysis was empirical rather than theoretical, it was difficult to determine whether PCA-based dimension reduction was optimal across all maximization problems.

\begin{algorithm} [t]
        \caption{FastPRIM for minimizing variance}\lb{FastPRIMII}
        \DontPrintSemicolon
	\KwIn{Elliptical random vector $\bfX \in \bbR^d$, peeling parameter $\beta \in (0,1)$.}
	\KwOut{$\giventhat*{\bfY}{\bigcap_{i=1}^k \left\{ Y_i \in \sfT_{Y_i}^\alpha \right\}}$.}
        Compute the orthogonal transformation $\bfY = \bfV' \bfX$, sorting its components by decreasing variance.\;
        Choose $k$ (say, by maximization of the spectral gap).\;
        Peel simultaneously a probability $\alpha/2 = (1-\beta)/(2k)$ of each of the tails of the $k$ leading principal components.\;
\end{algorithm}

To address the first problem, it was clear that the family of distributions needed to be constrained in order to allow peeling multiple tails simultaneously. Therefore, Díaz--Pachón, Rao, and Dazard proposed the fastPRIM algorithm for mode-hunting of a multivariate normal distribution, which first reduced dimensionality via PCA and then peeled in a single step the $2k$ tails of the preserved $k$ ($< d$) dimensions  \cite{DiazDazardRao2017}. To address the second problem of LSBH, Liu et al extended fastPRIM to hunt $\beta$-modes---rectangular boxes of probability $\beta$ where the volume of the selected box is minimized---to truncated multivariate normal and symmetric Laplace distributions \cite{LiuEtAl2023}, proving that the minimizer of the $k$-dimensional volume of $\beta$-modes is obtained for the centered box in the projection of the {\it pettiest components} (the eigenvectors of the covariance matrix with the smallest variance), not the leading principal components. This result was proven using active information (AIN) as the objective function, where AIN is defined as $\ain(\sfB) \defeq \log[\bfP(\sfB)/\bfP_0(\sfB)]$ \cite{DiazMarks2020a, DiazSaenzRao2020, DiazHossjer2022, HossjerDiazRao2022, ZhouEtAl2023, DiazHossjerMathew2024, HossjerEtAl2024, DiazEtAl2025, ChenDiaz2026}. In their analysis, $\bfP$ and $\bfP_0$ are the empirical and uniform distributions, respectively, over $\calX$. Thus, there is a local mode in $\sfB$ if $\ain(\sfB) > 0$ \cite{DiazEtAl2019}. However, generalization to more multivariate distributions was hindered by a counterexample (Remark 1 and Figure 2 of \cite{LiuEtAl2023}). 

When $\bfX$ is a random vector in the elliptical family, we propose Algorithms \ref{FastPRIMI} and \ref{FastPRIMII} for mode-hunting. Theorem \ref{TheTh} proves that Algorithm \ref{FastPRIMI} maximizes the total variance of the final truncated space. Moreover, Proposition \ref{Th2} proves that the volume of the $k$-dimensional peeled box in Algorithm \ref{FastPRIMI} is minimal.  In the other direction, Algorithm \ref{FastPRIMII} optimizes for modes hunted by variance minimization, as peeling the leading principal components reduces the total variance and the Frobenius norm the most (by Theorem \ref{TheTh}) while maximizing the volume of its $k$-dimensional peeled box of probability $\beta$. Theorem \ref{TheTh} and Proposition \ref{Th2} significantly extend Theorems 3 and 4 of \cite{LiuEtAl2023}, where it was shown that the last principal components (referred to as pettiest components) minimize the volume of $\sfT_{I_k, \tilde Y}$ when $\bfX$ is either multivariate normal or symmetric multivariate Laplace: first, the minimization part of Proposition \ref{Th2} generalizes this finding to all random vectors with elliptical distributions; second, Theorem \ref{TheTh} shows that this is equivalent to maximizing the total variance (and the Frobenius norm) when the points in the truncated $k$ dimensional box are sent back to the whole $d$-dimensional space, making Algorithm \ref{FastPRIMI} stable. The maximization part demonstrates that the volume of $\sfT_{I_k, \tilde Y}$ is maximized when the $k$ leading principal components are truncated by peeling $\alpha/2$ from each of their tails, therefore optimizing for mode-hunting via variance minimization---a fact that Algorithm \ref{FastPRIMII} exploits.

\begin{figure}[t]
\centering
{\includegraphics[width=4cm,height=3.6cm]{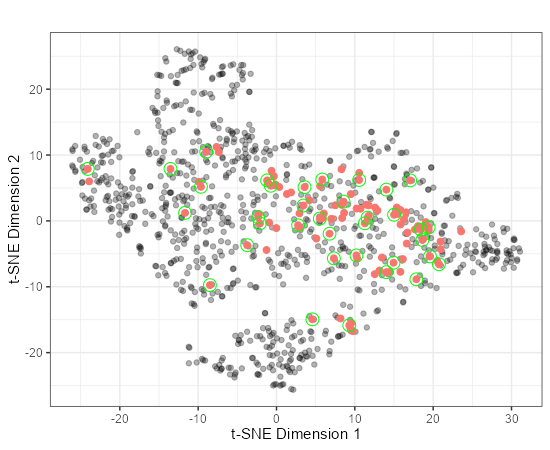}}\quad\raisebox{0.0\height}
{\includegraphics[width=4cm,height=3.6cm]{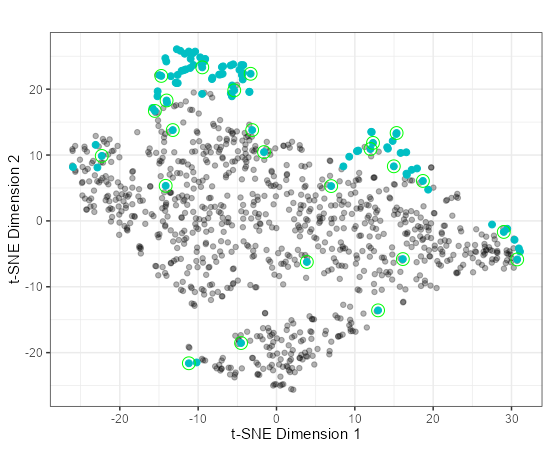}}
\\[1em]
{\includegraphics[width=4cm,height=4.4cm]{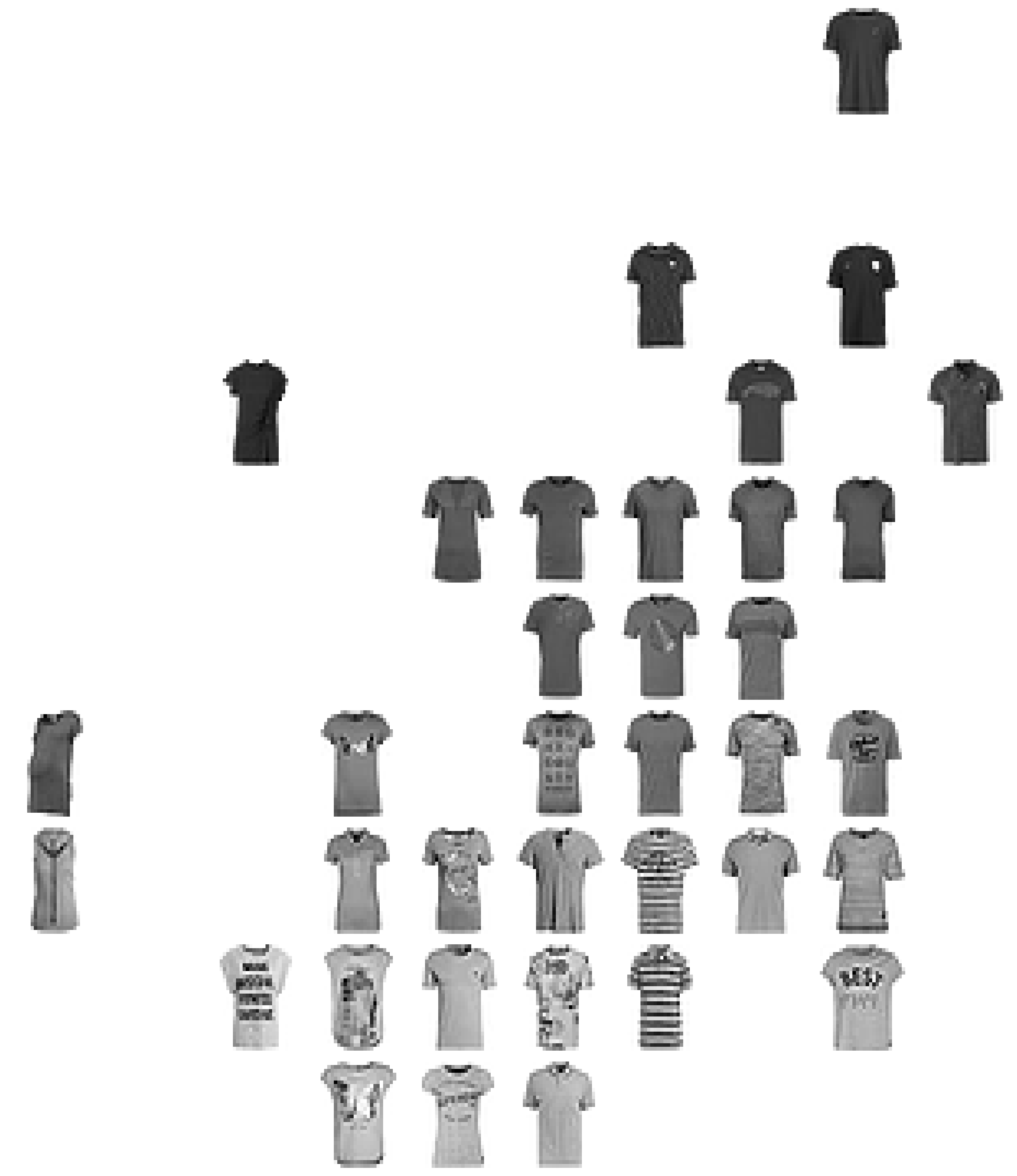}}\quad\raisebox{0.0\height}
{\includegraphics[width=4cm,height=4.4cm]{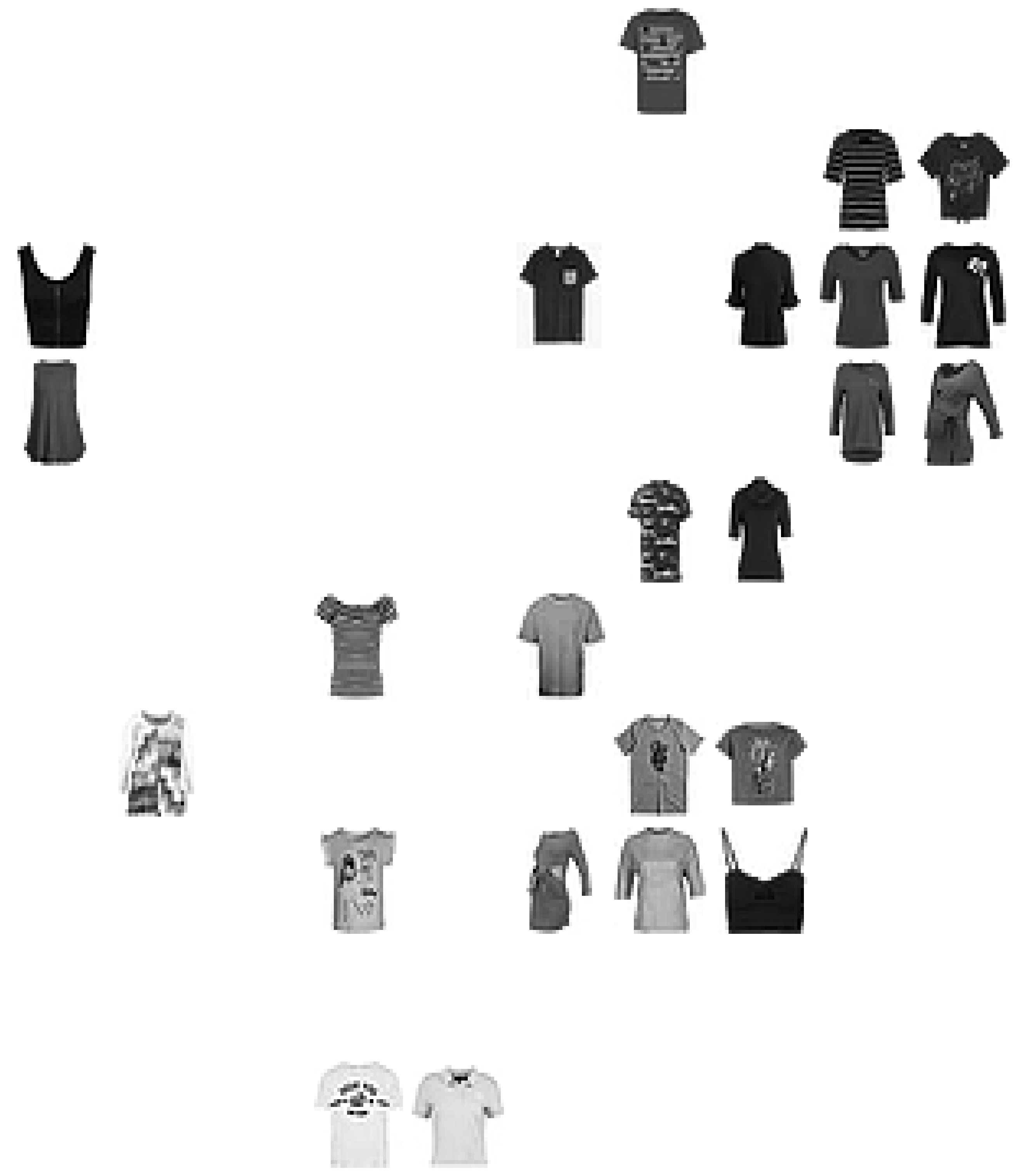}}
  \caption{T-shirt example. Left) peeling principal directions. Right) peeling pettiest
  directions} \label{t-shirt}
\end{figure}

\section{No Free Lunch}

With the established proportionality between bounding-box volume and variance, we notice a clear symmetry between the two objectives. Peeling the pettiest directions maximizes the preserved variance, corresponding to a smaller bounding-box volume in $k$ dimensions; while peeling the principal directions minimizes the preserved variance, thus maximizing the bounding-box volume in $k$ dimensions. Essentially, these two objectives are dual versions of bump-hunting problems, differing only in the direction of optimization. This observation can be complemented by a finite-domain No-Free-Lunch result: once the admissible peeling rules are restricted to the finite set of principal-component coordinate subsets, the classical NFLTs for finite optimization apply directly.

\begin{proposition}[No free lunch for dimension selection]\lb{NFLT}
Fix $k \in \{1,\ldots,d\}$ and $\alpha \in (0,1)$ with $k\alpha \leq 1$. Let $\calI_k:=\{I \subset [d]: \abs{I}=k\}$ be the finite set of admissible dimension selection strategies in the eigenbasis $\bfV$, let $\calY$ be a finite set, and let $\calF:=\calY^{\calI_k}$ be the set of all objective functions $f:\calI_k \to \calY$. Let $a$ be a deterministic non-revisiting search algorithm on $\calI_k$, and let $Y(f,m,a)$ denote the sequence of objective values observed by $a$ after $m$ queries. Then for any two such algorithms $a_1$ and $a_2$, and for every $y \in \calY^m$,
\[
\bigl|\{f \in \mathcal F : Y(f,m,a_1)=y\}\bigr|
=
\bigl|\{f \in \mathcal F : Y(f,m,a_2)=y\}\bigr|.
\]
Hence, averaged over all objective functions in $\calF$, no dimension selection algorithm has an a priori advantage over any other.
\end{proposition}

Proposition \ref{NFLT} addresses the class of all finite-valued objective functions on the finite strategy space $\calI_k$, not just the specific geometric objectives studied in Section \ref{S:Theory}. It shows that, absent structural assumptions on the objective, no dimension selection rule is universally preferred. By contrast, Theorem \ref{TheTh} and Proposition \ref{Th2} identify such a structure for the objectives of optimizing either preserved variance or bounding-box volume. Within this structured subclass, choosing the pettiest and principal components optimizes different scientifically meaningful criteria, as explained in the last paragraph of Section \ref{S:Algos}.  This is in sharp contrast to supervised learning with cross-validation as the objective function, where ``anti-cross-validation'' performs equally well under the NFLTs but lacks comparable scientific motivation.

\begin{figure}[t]
\centering
{\includegraphics[width=4cm,height=3.6cm]{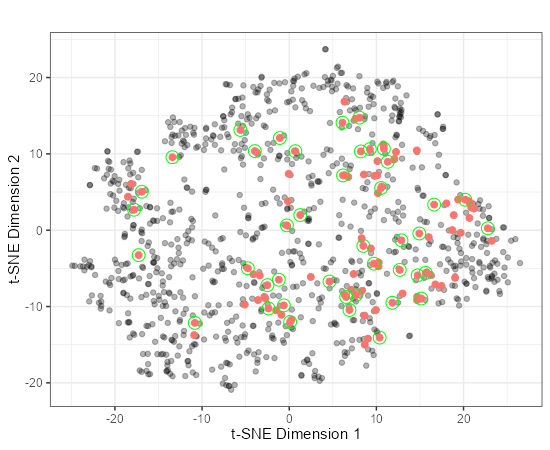}}\quad\raisebox{0.0\height}
{\includegraphics[width=4cm,height=3.6cm]{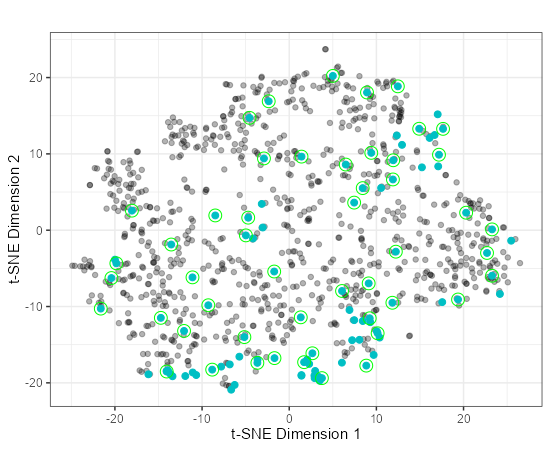}}
\\[1em]
{\includegraphics[width=4cm,height=4.4cm]{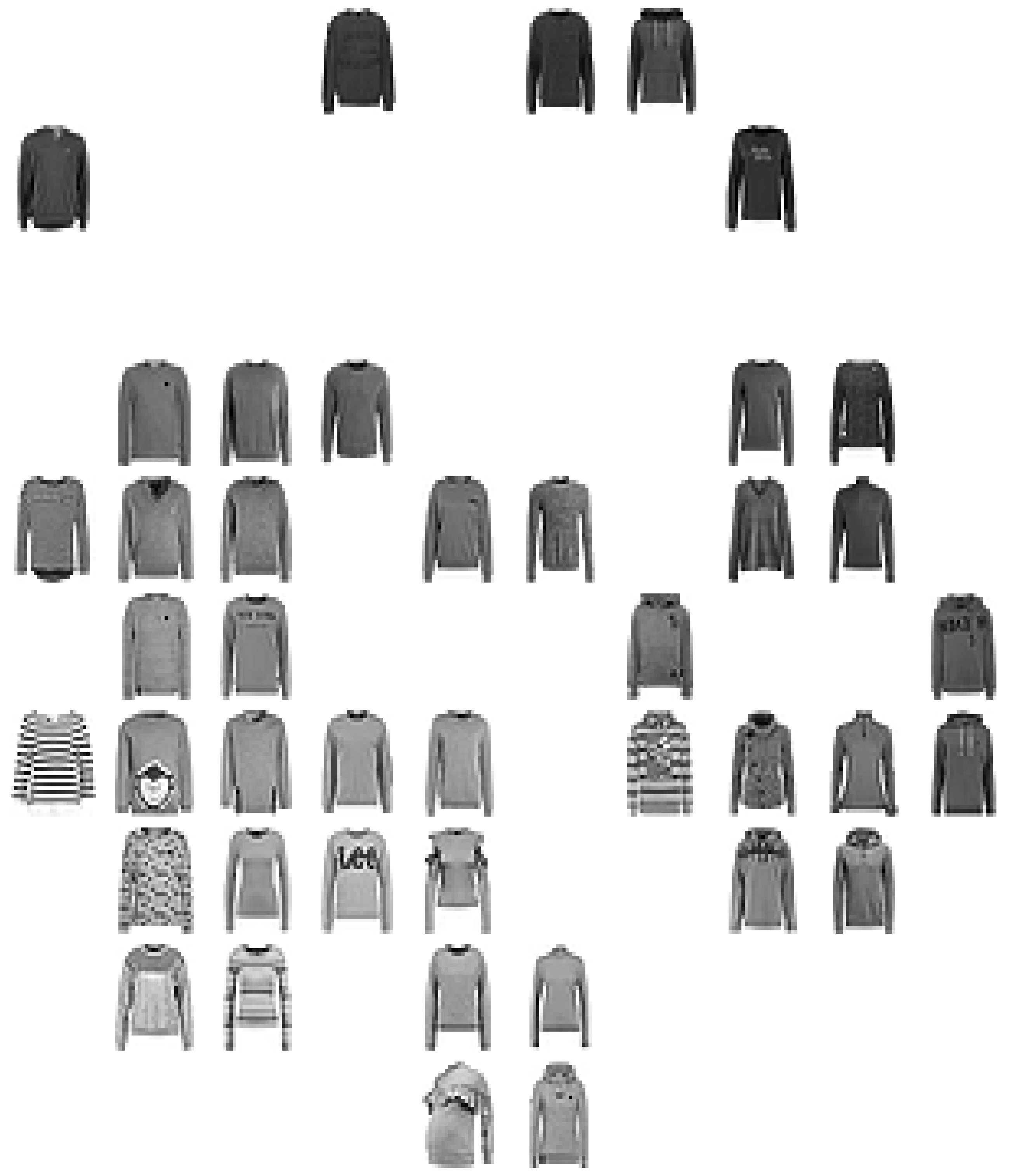}}\quad\raisebox{0.0\height}
{\includegraphics[width=4cm,height=4.4cm]{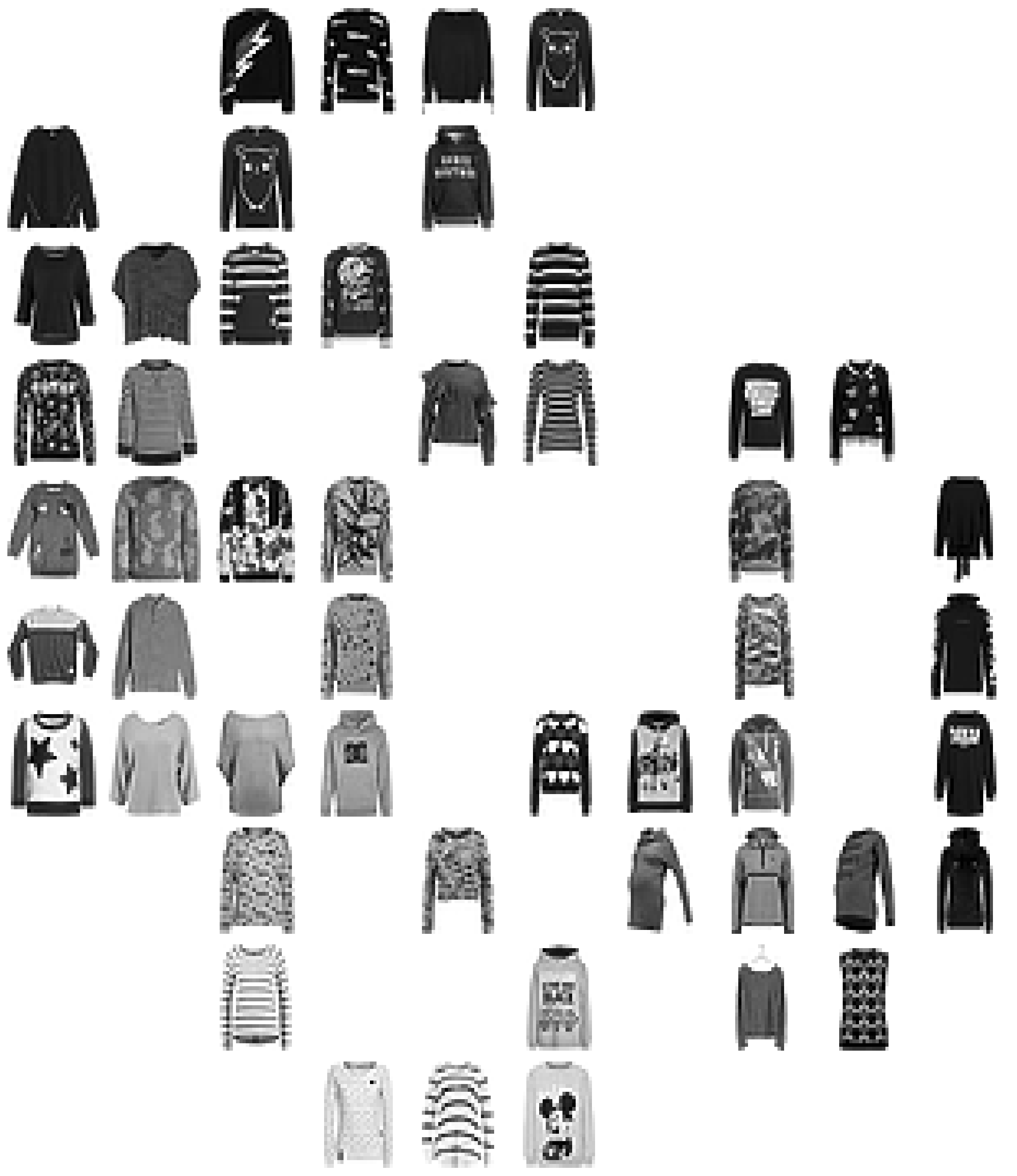}}
  \caption{Pullover example. Left) peeling principal directions. Right) peeling pettiest
  directions} \label{pullover}
\end{figure}

\section{Fashion MNIST Example}\lb{Example}

We demonstrate our algorithms on the Fashion MNIST data set, with code available at GitHub (\url{https://github.com/TH20255/PRIM-Fashion}). Because our methods assume unimodal data, we first classify the images by label and carry out the analysis separately on each class. Prior to peeling, we apply PCA and identify the pettiest components using the \emph{log spectral gap} criterion (see Section~\ref{sec:lsg}): the pettiest region is defined as the $k$ consecutive principal components that straddle the largest drop in the log-scale scree plot. The principal
region comprises the leading $k$ components. We then execute Algorithms \ref{FastPRIMI} and \ref{FastPRIMII}.

\begin{figure}[t]
\centering
\includegraphics[width=\columnwidth]{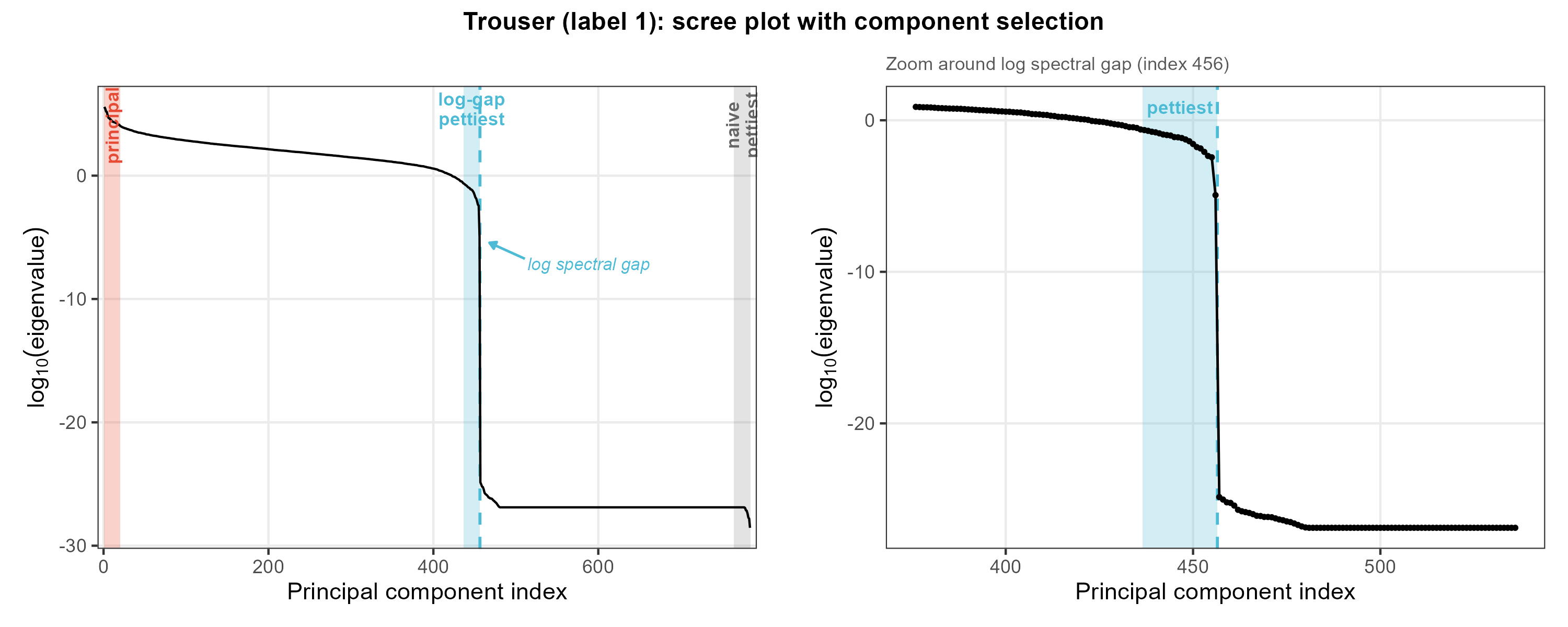}
\caption{Log-scale scree plot for the trouser class. Red band: principal
  region (top 20 PCs). Blue band: pettiest region selected by the log
  spectral gap (PCs~437--456). Gray band: naive tail selection (PCs~765--784).
  The dashed vertical line marks the log spectral gap at index~456.
  The right panel zooms into the neighborhood of the gap.}
\label{fig:scree_trouser}
\end{figure}

The results can be interpreted as follows. Peeling the principal directions identifies the mode---the most common clothing styles within each class---while peeling the pettiest directions preserves the greatest variability, offering a representative overview of the diverse styles present in the data. For instance, Fig.\@~\ref{t-shirt} displays results for t-shirts, highlighting the selected points from the principal subset (in red) and the pettiest subset (in blue) on a t-SNE plot, and also showcasing the grid-sampled results (marked with green circles). We overlay a $10\times10$ grid onto the t-SNE projection and, for each occupied cell, randomly select one representative image.

The results vary meaningfully across categories. \textbf{T-shirts/tops} (Fig.\@~\ref{t-shirt}): the principal subset concentrates on plain, dark-coloured basic crew-neck tees, while the pettiest subset captures a wide variety of styles---cropped tops, tank
tops, long-sleeve shirts, and heavily printed or striped garments. \textbf{Pullovers} (Fig.\@~\ref{pullover}): the principal subset
selects plain, solid-coloured sweatshirts and basic long-sleeve tops, whereas the pettiest subset is dominated by bold graphic prints, large logos, character motifs, and patterned knits---a visually striking contrast. The remaining categories are analyzed in Appendix \ref{CntEx}. The t-SNE visualization confirms that the principal subset (red) is more concentrated, while the pettiest subset (blue) appears more dispersed, capturing finer or less common distinctions. Examining both results together provides valuable insight into the garment market structure and can inform business decisions. 

Table~\ref{tab:prim_statistics} reports the bootstrap estimates (and standard errors) of the preserved variance statistics for each label, based on $B=1{,}000$ replicates. Quantitatively, the pettiest subset consistently yields a substantially larger total variance than the principal subset---by a factor of roughly $1.6$--$4.2\times$ across all ten categories---confirming the theoretical prediction of Theorem~\ref{TheTh}. The Frobenius norm and operator norm follow the same ordering. The log generalized variance, which by
Proposition~\ref{Counter} is invariant to the choice of dimensions under the elliptical model, shows only modest differences between the two subsets in practice, reflecting that the theoretical invariance is approximate for finite samples and non-elliptical data. Bootstrap standard errors are uniformly small relative to the point estimates, attesting to the stability of the procedure.


\subsection{The role of log spectral gap selection}\label{sec:lsg}

A key practical question in applying Algorithms~\ref{FastPRIMI} and~\ref{FastPRIMII} is how to identify the pettiest component region. A natural but naive approach is to simply take the last $k$ eigenvectors of the sample covariance matrix, irrespective of where in the eigenvalue spectrum they fall. We show here that this naive choice can be uninformative: when the chosen tail includes components that still carry a genuine signal, the resulting pettiest box becomes indistinguishable from a near-uniform sample of the data, and the variance contrast between the principal and pettiest boxes collapses.

\begin{figure}[t]
\centering
{\includegraphics[width=4cm,height=3.6cm]{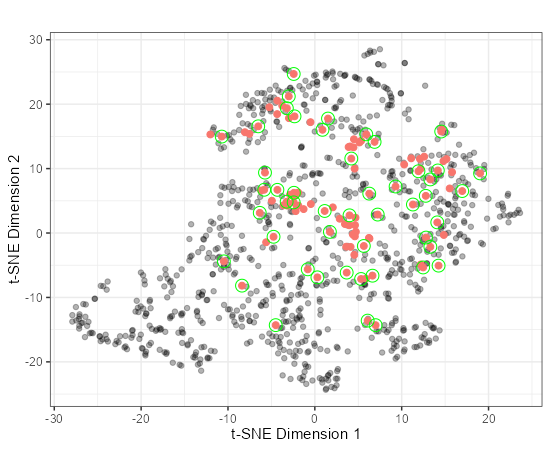}}%
\quad
\raisebox{0.0\height}{%
  \includegraphics[width=4cm,height=3.6cm]{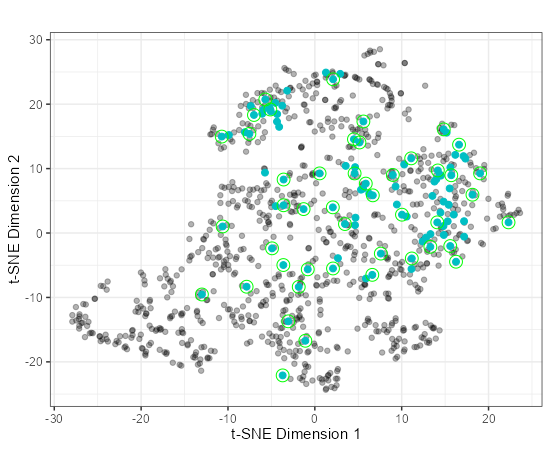}}
\\[1em]
{\includegraphics[width=4cm,height=4.4cm]{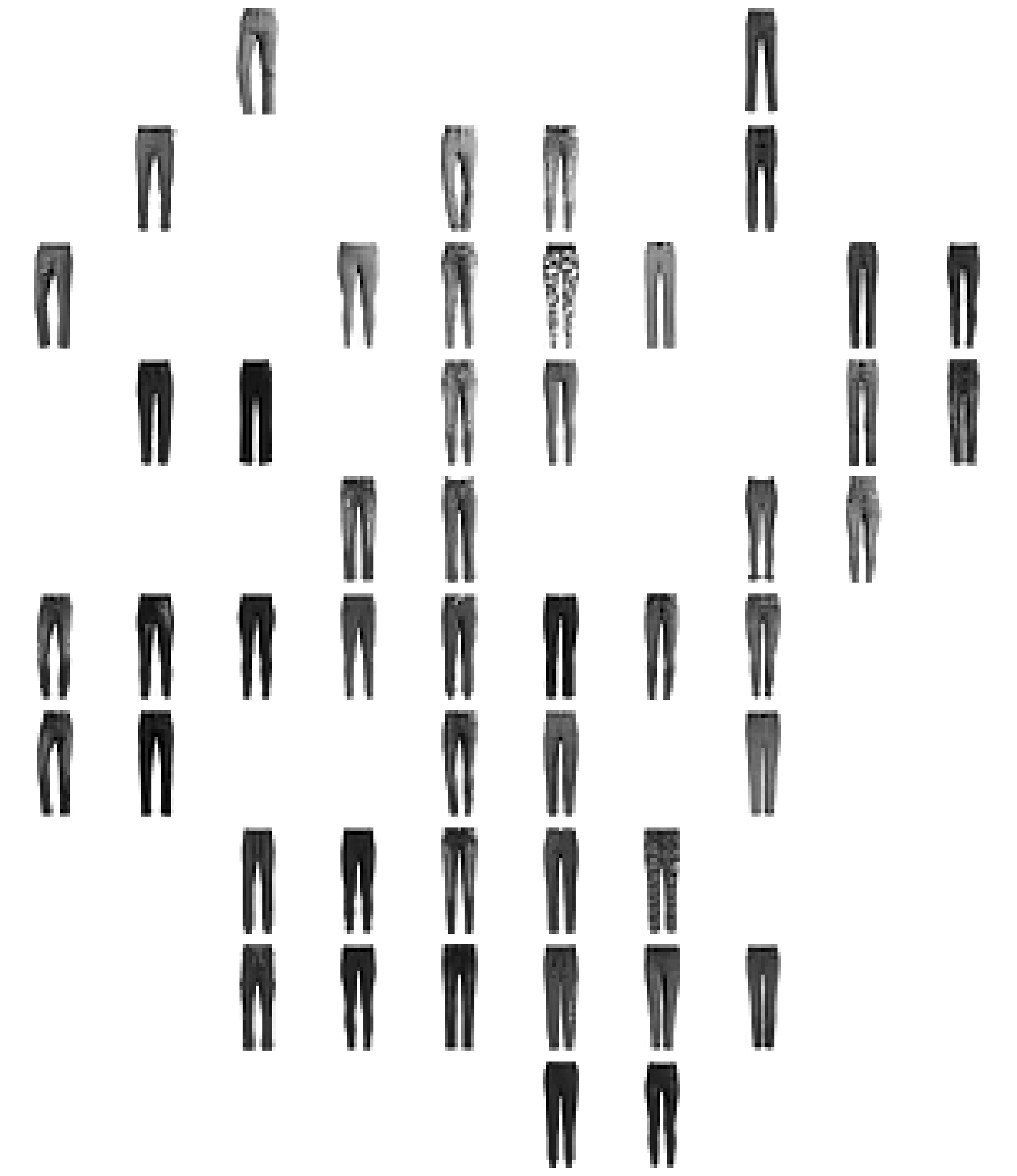}}%
\quad
\raisebox{0.0\height}{%
  \includegraphics[width=4cm,height=4.4cm]{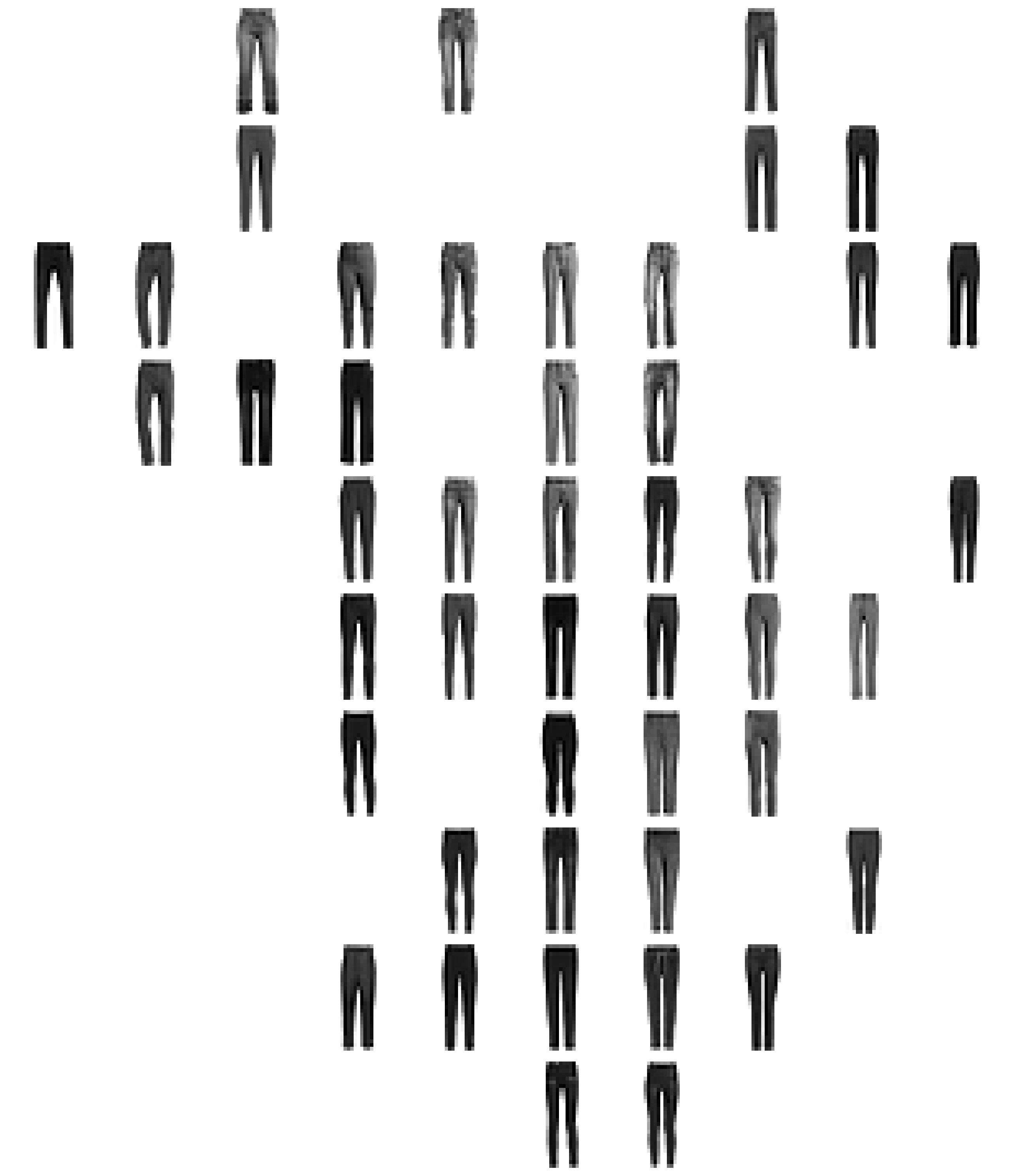}}
\caption{Trouser class with \emph{naive} pettiest selection (raw tail components, no log spectral gap). Left) principal subset. Right) pettiest subset. The pettiest subset is visually indistinguishable from a uniform sample, with no discernible structure separating it from the full data.}\label{trouser_naive}
\end{figure}

\subsubsection*{The log spectral gap criterion}

Let $\hat\lambda_1 \ge \cdots \ge \hat\lambda_p$ be the sample eigenvalues. The \emph{log spectral gap} is defined as the index
\[
  j^\ast = \arg\max_{1 \le j \le p-1}
           \bigl(\log\hat\lambda_j - \log\hat\lambda_{j+1}\bigr),
\]
i.e.\ the position of the largest drop in the log-scale scree plot. The pettiest region is then the $k$ components immediately below $j^\ast$: $\{\hat\lambda_{j^\ast-k+1}, \ldots, \hat\lambda_{j^\ast}\}$, and the principal region comprises the top $k$ components $\{\hat\lambda_1, \ldots, \hat\lambda_k\}$.  Working on the log scale is essential: on the original scale, the gap is dominated by the top eigenvalues, which merely reflects the overall magnitude of variation rather than the structural break that separates signal from noise \cite{AhnHorenstein2013}.

\subsubsection*{Empirical consequence: naive tails give near-uniform sampling}

Fig.\@~\ref{fig:scree_trouser} shows the $\log_{10}$-scale scree plot for the trouser class.  The spectrum exhibits a sharp structural break at principal component~456, where the log-scale eigenvalue drops by more than two orders of magnitude over a narrow band of components (indices~437--456, highlighted in blue).  A naive tail selection, by contrast, picks the last 20 components (indices~765--784, gray band),which lie in the flat noise floor where all eigenvalues are nearly equal.

To illustrate the practical consequence, Fig.\@~\ref{trouser_naive} shows the grid-sampled results obtained with naive pettiest selection (the last $k$ raw-spectrum components). The pettiest subset is visually indistinguishable from a near-uniform draw from the data: images span the full gamut of trouser styles without discernible structure, while the principal subset correctly
concentrates on the most common trouser designs.

This observation is corroborated numerically. With naive selection, the pettiest total variance for trousers is $8.61\times10^5$, only $24\%$ above the principal value of $6.92\times10^5$---a negligible contrast. Under the log spectral gap criterion, the same statistics become $2.81\times10^6$ versus $6.74\times10^5$, a ratio of $4.2\times$, consistent with the theoretical prediction. Similar collapses are observed across other categories with the naive approach: for sneakers, the pettiest-to-principal ratio is $1.32$ with naive selection versus $2.15$ with log spectral gap; for ankle boots, it is $1.38$ versus $2.50$.

The explanation is straightforward: the very last eigenvalue components are dominated by numerical noise and image artifacts rather than any garment-specific structure. Because PRIM searches for a high-density box in this near-isotropic subspace, it essentially recovers a uniform sample. The log spectral gap provides exactly that structural knowledge by identifying the natural
break between the informative and the noise-dominated portion of the spectrum.

\begin{table*}[t]
\tiny
\centering
\caption{Bootstraps estimations and standard errors by label.Statistics are evaluated on the points selected by the final box }
\label{tab:prim_statistics}
\begin{tabular}{llcccc}
  \toprule
Clothing Type & Model & Total variance & Frobenius norm
              & $\log$ generalized variance & Operator norm \\
  \midrule
T-shirt/top & Principal
  & $1.26\times10^{6}$ ($2.57\times10^{3}$)
  & $4.83\times10^{5}$ ($1.61\times10^{3}$)
  & $541.71$ ($1.04$)
  & $4.44\times10^{5}$ ($1.76\times10^{3}$) \\
 & Pettiest
  & $3.91\times10^{6}$ ($5.94\times10^{3}$)
  & $1.49\times10^{6}$ ($4.05\times10^{3}$)
  & $615.92$ ($1.17$)
  & $1.33\times10^{6}$ ($4.59\times10^{3}$) \\
  \addlinespace[2pt]\midrule\addlinespace[2pt]
Pullover & Principal
  & $1.48\times10^{6}$ ($2.76\times10^{3}$)
  & $5.25\times10^{5}$ ($2.02\times10^{3}$)
  & $552.26$ ($1.07$)
  & $4.52\times10^{5}$ ($2.32\times10^{3}$) \\
 & Pettiest
  & $3.91\times10^{6}$ ($5.69\times10^{3}$)
  & $1.35\times10^{6}$ ($3.95\times10^{3}$)
  & $622.44$ ($1.18$)
  & $1.14\times10^{6}$ ($4.74\times10^{3}$) \\
  \addlinespace[2pt]\midrule\addlinespace[2pt]
Shirt & Principal
  & $1.36\times10^{6}$ ($2.74\times10^{3}$)
  & $4.23\times10^{5}$ ($1.49\times10^{3}$)
  & $557.03$ ($1.05$)
  & $3.25\times10^{5}$ ($1.70\times10^{3}$) \\
 & Pettiest
  & $3.54\times10^{6}$ ($4.81\times10^{3}$)
  & $1.28\times10^{6}$ ($3.26\times10^{3}$)
  & $621.56$ ($1.23$)
  & $1.12\times10^{6}$ ($3.54\times10^{3}$) \\
  \addlinespace[2pt]\midrule\addlinespace[2pt]
Dress & Principal
  & $1.43\times10^{6}$ ($2.26\times10^{3}$)
  & $4.04\times10^{5}$ ($1.10\times10^{3}$)
  & $563.31$ ($1.07$)
  & $2.95\times10^{5}$ ($1.25\times10^{3}$) \\
 & Pettiest
  & $2.74\times10^{6}$ ($4.34\times10^{3}$)
  & $9.90\times10^{5}$ ($2.49\times10^{3}$)
  & $585.93$ ($1.11$)
  & $7.46\times10^{5}$ ($2.72\times10^{3}$) \\
  \addlinespace[2pt]\midrule\addlinespace[2pt]
Coat & Principal
  & $1.83\times10^{6}$ ($2.61\times10^{3}$)
  & $6.79\times10^{5}$ ($1.71\times10^{3}$)
  & $560.00$ ($1.06$)
  & $5.65\times10^{5}$ ($2.01\times10^{3}$) \\
 & Pettiest
  & $3.28\times10^{6}$ ($4.12\times10^{3}$)
  & $1.11\times10^{6}$ ($2.64\times10^{3}$)
  & $607.88$ ($1.18$)
  & $8.84\times10^{5}$ ($3.36\times10^{3}$) \\
  \addlinespace[2pt]\midrule\addlinespace[2pt]
Sandal & Principal
  & $1.35\times10^{6}$ ($2.02\times10^{3}$)
  & $3.21\times10^{5}$ ($7.51\times10^{2}$)
  & $575.03$ ($1.15$)
  & $2.14\times10^{5}$ ($9.01\times10^{2}$) \\
 & Pettiest
  & $3.48\times10^{6}$ ($5.27\times10^{3}$)
  & $9.16\times10^{5}$ ($2.12\times10^{3}$)
  & $629.17$ ($1.28$)
  & $5.95\times10^{5}$ ($2.26\times10^{3}$) \\
  \addlinespace[2pt]\midrule\addlinespace[2pt]
Ankle boot & Principal
  & $1.39\times10^{6}$ ($2.82\times10^{3}$)
  & $4.56\times10^{5}$ ($1.65\times10^{3}$)
  & $553.11$ ($1.07$)
  & $3.80\times10^{5}$ ($1.86\times10^{3}$) \\
 & Pettiest
  & $3.46\times10^{6}$ ($3.46\times10^{3}$)
  & $1.17\times10^{6}$ ($2.23\times10^{3}$)
  & $609.40$ ($1.18$)
  & $9.46\times10^{5}$ ($2.77\times10^{3}$) \\
   \addlinespace[2pt]\midrule\addlinespace[2pt]
Bag & Principal
  & $2.60\times10^{6}$ ($3.20\times10^{3}$)
  & $1.08\times10^{6}$ ($2.00\times10^{3}$)
  & $594.49$ ($1.16$)
  & $1.00\times10^{6}$ ($2.14\times10^{3}$) \\
 & Pettiest
  & $4.10\times10^{6}$ ($5.05\times10^{3}$)
  & $1.37\times10^{6}$ ($2.80\times10^{3}$)
  & $629.72$ ($1.25$)
  & $1.08\times10^{6}$ ($3.55\times10^{3}$) \\
  \addlinespace[2pt]\midrule\addlinespace[2pt]
Sneaker & Principal
  & $9.87\times10^{5}$ ($1.61\times10^{3}$)
  & $2.87\times10^{5}$ ($9.54\times10^{2}$)
  & $542.36$ ($1.06$)
  & $2.20\times10^{5}$ ($1.19\times10^{3}$) \\
 & Pettiest
  & $2.12\times10^{6}$ ($3.54\times10^{3}$)
  & $8.29\times10^{5}$ ($2.09\times10^{3}$)
  & $563.13$ ($1.05$)
  & $7.43\times10^{5}$ ($2.27\times10^{3}$) \\
  \addlinespace[2pt]\midrule\addlinespace[2pt]
Trouser & Principal
  & $6.74\times10^{5}$ ($1.31\times10^{3}$)
  & $2.38\times10^{5}$ ($6.63\times10^{2}$)
  & $480.80$ ($0.90$)
  & $1.74\times10^{5}$ ($7.29\times10^{2}$) \\
 & Pettiest
  & $2.81\times10^{6}$ ($5.96\times10^{3}$)
  & $1.33\times10^{6}$ ($3.58\times10^{3}$)
  & $548.01$ ($0.98$)
  & $1.22\times10^{6}$ ($3.75\times10^{3}$) \\
  \bottomrule
\end{tabular}
\end{table*}

\section{Conclusion}\lb{Conclusions}
We have proposed an integrated framework for extracting desired subsets from high-dimensional data. Under regularity conditions---specifically, when the distribution of the data belongs to the elliptical family---we demonstrated that peeling along the pettiest components maximizes preserved variance, while peeling along the principal component minimizes the $k$-dimensional bounding-box volume. Our analysis reveals a natural duality between these two legitimate mode-hunting objectives, which naturally invites a No-Free-Lunch interpretation. Moreover, we addressed concerns regarding the stability of peeling the pettiest components because Theorem \ref{TheTh} shows that the preserved total variance and Frobenius norm depend mainly on the unpeeled leading principal components. Our demonstration on the Fashion MNIST data set produced striking results, revealing both prevailing trends and the diversity of available designs.

Our method relies on several assumptions. First, the theoretical guarantees assume that projected features follow an elliptical distribution; performance on heavily skewed data is unknown. Second, the fastPRIM peeling uses marginal quantiles and may fail to handle multiple modes. Finally, our experiments focus on grayscale image data of moderate resolution; extensions to color, higher‑resolution, or non‑image data will require additional validation. Addressing these limitations constitutes promising future work.

\appendices

\section{Auxiliary results for the principal theorems}\lb{Proofs}

In this Appendix, we prove a series of auxiliary results leading to the proof of Theorem \ref{TheTh} and Proposition \ref{Th2}. When $k=1$ and the optimization problem on the total variance is restricted to the eigenbasis transformation, we begin by proving Theorem \ref{Normality}, a particular case of the main result that assumes the independence of the principal components. Then, using Lemma \ref{Cheb}, we prove Theorem \ref{Th1}, an extension of Theorem \ref{Normality} to all elliptical distributions. As a direct implication of Theorem \ref{Th1}, we obtain the important Corollary \ref{CondCovM}, a result that explicitly derives the form of the preserved covariance matrix. We then use Corollary \ref{CondCovM} to extend Theorem \ref{Th1} to all $k$-dimensional subspaces of all orthogonal transformations $\bfP$ of $\bfX$ (Theorem \ref{GenTh1}), and to other variance measures like the Frobenius norm, the generalized variance, and the operator norm (Theorem \ref{Gen}). Finally, with all the background in place, we prove Theorem \ref{TheTh} and Proposition \ref{Th2}.

\begin{theorem}\lb{Normality}
Let $\bfX \sim {\calN}({\bf 0}, \boldsymbol{\Sigma})$ be a $d$-dimensional normal random vector with eigenvalues $\lambda_1 \geq \cdots \geq \lambda_d$ and corresponding orthonormal eigenvectors $v_1, \ldots, v_d$. Let the random vector $\bfY = (Y_1, \ldots, Y_d) = \bfV'\bfX$ be the rotation of $\bfX$ in the direction of the eigenbasis of $\bfSigma$, where $\bfV = (v_1, \ldots, v_d)$. For $\alpha \in [0,1]$, define the peeled box 
\[ \sfT_{Y_i}^\alpha \defeq \left[Q_{Y_i}(\alpha/2), Q_{Y_i}(1-\alpha/2)\right] \]
as the interquantile range between the $\alpha/2$ and $(1-\alpha/2)$ quantiles of $Y_i$. Then, the preserved total variance in the peeled box $\mathrm{Tr} \left(\bfLambda_{i,\alpha}\right)$ is maximized by peeling $Y_d$ and minimized by peeling $Y_1$.
\end{theorem}

\begin{proof}[Proof of Theorem \ref{Normality}]
Let $Z \sim \calN(0,1)$ be a standard normal random variable. Then $Y_j = \sqrt{\lambda_j}Z$, for $j = \{1, \ldots, d\}$, and 
\begin{align}\lb{NormalBound}
	\Var\giventhat[\big]{Y_i}{Y_i \in \sfT_{Y_i}^\alpha} = \lambda_i \Var\giventhat[\big]{Z}{Z \in \sfT_Z^\alpha} < \lambda_i,
\end{align}
where the inequality in \eqref{NormalBound} is obtained because $\Var\giventhat{Z}{Z \in \sfT_Z^\alpha} < \Var(Z) = 1$. Therefore, calling $\Delta_i$ the total variance removed by peeling along $v_i$,
\begin{align}\lb{Indep}
\begin{aligned}
	\Delta_i &\defeq \mathrm{Tr}(\Cov(\bfY)) - \mathrm{Tr} \left(\Cov\giventhat[\big]{\bfY}{Y_i \in \sfT_{Y_i}^\alpha}\right) \\
		&= \mathrm{Tr}(\bfLambda) - \mathrm{Tr}(\bfLambda_{i,\alpha}) \\
		&= \sum_{j=1}^d \lambda_j - \left(\sum_{j \ne i} \lambda_j + \lambda_i\Var\giventhat[\big]{Z}{Z \in \sfT_Z^\alpha} \right)\\
		&= c\lambda_i,
\end{aligned}
\end{align}
where the second equality of \eqref{Indep} used \eqref{NormalBound} and independence of the principal components under normality, and $c \in (0,1)$ is an absolute constant. Therefore, $\Delta_i$ is minimized when $i=d$ and maximized when $i=1$, proving the result.
\end{proof}

A generalization of Theorem \ref{Normality} that assumes only uncorrelatedness rather than independence does not allow the application of \eqref{Indep}. A different route must be explored.

\begin{lemma}\lb{Cheb}
	For $f,g$ two real measurable functions on $\calX$, where $f$ is increasing and $g$ is decreasing, $\bfE [f(X)g(X)] \le \bfE f(X) \bfE g(X)$.
\end{lemma}
\begin{proof}
Let $Y$ be an independent copy of $X$. Then,
\begin{align*}
	 0 &\ge \bfE\{f(X) - f(Y)][g(X) - g(Y)]\} \\
	 &= \bfE [f(X)g(X) - f(X)g(Y) - f(Y)g(X) +f(Y)g(Y)] \\
	 &= 2\bfE [f(X) g(X)] -2\bfE f(X) \bfE g(X),
\end{align*}
which proves the result.
\end{proof}

\begin{theorem}\lb{Th1} 
Let $\bfX \in \calX \subset \bbR^d$ be a random vector in the elliptic family with covariance matrix $\bfSigma$ having eigenvalues $\lambda_1 \geq \cdots \geq \lambda_d$ and corresponding orthonormal eigenvectors $v_1, \ldots, v_d$. Let the random vector $\bfY = (Y_1, \ldots, Y_d) = \bfV' \bfX$ be the rotation of $\bfX$ in the direction of the eigenbasis of $\bfSigma$, where $\bfV = (v_1, \ldots, v_d)$. For $\alpha \in [0,1]$, define the peeled box 
\[ \sfT_{Y_i}^\alpha \defeq \left[Q_{Y_i}(\alpha/2), Q_{Y_i}(1-\alpha/2)\right] \]
as the interquantile range between the $\alpha/2$ and $(1-\alpha/2)$ quantiles of $Y_i$. Then, the preserved total variance in the peeled box $\mathrm{Tr} \left(\Cov\giventhat[\big]{\bfY}{Y_i \in \sfT_{Y_i}^\alpha}\right)$ is maximized by peeling $Y_d$ and minimized by peeling $Y_1$.
\end{theorem}

\begin{proof}
Assume without loss of generality that the mean of $\bfX$ is zero, as the result is location-invariant. By \eqref{StochRepED}, $\bfX$ admits the stochastic representation 
\begin{align}\lb{SRTh}
	\bfX = R \bfA \bfO, 
\end{align}
where $\bfO = (O_1, \dots, O_d)$ is a random vector uniformly distributed on the unit sphere $\bbS^{d-1}$, $R \geq 0$ is a nonnegative random variable representing the radius and independent of $\bfO$, and $\bfA$ is a matrix such that $\bfA\bfA' = \bfSigma$. In the principal coordinate system, 
\begin{align}\lb{PCRot}
	\bfY = \bfV'\bfX = \bfV'(R \bfA \bfO) = R \left(\bfV' \bfA\right) \bfO = R \bfLambda^{1/2} \bfO, 
\end{align}
where the second equality follows from \eqref{SDT} and the definition of $\bfA$. Then, 
\begin{align}\lb{CondCov}
\begin{aligned}
	\Tr\left(\bfLambda_{i,\alpha}\right) &= \bfE\giventhat*{\|\bfY\|^2}{Y_i \in \sfT_{Y_i}^\alpha}\\
		&= \bfE\giventhat*{R^2 \sum_{j=1}^d \lambda_j O_j^2}{\lvert Y_i \rvert \leq q_i} \\
		&= \bfE\giventhat*{R^2 \sum_{j=1}^d \lambda_j O_j^2}{R \lvert O_i \rvert \leq q},
\end{aligned}
\end{align} 
where the first equality is obtained because $\bfY \mid Y_i \in \sfT_{Y_i}^\alpha$ retains the symmetry around 0; the second because of \eqref{SRTh} after defining $q_i \defeq Q_{Y_i}(1 - \alpha/2)$; and the third because $q_i = \sqrt{\lambda_i} \, q$ with $q = Q_{|S|}(1 - \alpha/2)$ fixed across $i$, where $S \stackrel{d}{=} Y_j/\sqrt{\lambda_j}$ for all $j$, by \eqref{PCRot} and symmetry of $\bfO$. By the law of total expectation and independence of $R$ and $\bfO$, \eqref{CondCov} becomes
\begin{align}\lb{CondCov2}
\begin{aligned}
	&\Tr\left(\bfLambda_{i,\alpha}\right) \\
		&= \bigintsss_0^1 \bfE\giventhat*{R^2 \sum_{j=1}^d \lambda_j O_j^2}{R \omega \leq q} f_{\lvert O_i \rvert}(\omega)\d \omega \\
		&= \frac{\int_0^1 \bfE\giventhat*{\sum_{j=1}^d \lambda_j O_j^2}{\lvert O_i \rvert = \omega}  \int_0^{q/\omega} r^2 F_R(\d r) f_{\lvert O_i \rvert}(\omega)\d \omega}{\int_0^1 F_R(q/\omega) f_{\lvert O_i \rvert}(\omega)\d \omega},
\end{aligned}
\end{align} 
where $F_R(\cdot)$ is the distribution of $R$ and $f_{\lvert O_i \rvert}(\cdot)$ is the density of $\lvert O_i \rvert$. Now, 
\begin{align}\lb{ConExp}
\begin{aligned}
	\bfE\giventhat*{\sum_{j=1}^d \lambda_j O_j^2}{\lvert O_i \rvert = \omega} 
	&= \lambda_i\omega^2 + \frac{\operatorname{Tr}(\bfSigma) - \lambda_i}{d-1}(1-\omega^2)\\
	&= \lambda_i \left(\frac{d\omega^2-1}{d-1}\right) + \operatorname{Tr}(\bfSigma)\frac{1-\omega^2}{d-1}
\end{aligned}
\end{align}
where the first equality in \eqref{ConExp} is obtained because $1 = \bfE \| \bfO \|_2^2 = \bfE \sum_{j=1}^d O_j^2 = \sum_{j=1}^d \bfE \|O_j \|_2^2$, and $\bfE\| O_j \|_2^2 = 1/d$ because $O_1, \ldots, O_d$ are identically distributed. Then, inserting \eqref{ConExp} into \eqref{CondCov2},
\begin{align}\lb{CondCov3}
\begin{aligned}
	\Tr\left(\bfLambda_{i,\alpha}\right) 
	&=\lambda_i \frac{d I_2 - I_1}{(d-1)D} + \operatorname{Tr}(\bfSigma)\frac{I_1-I_2}{(d-1)D}	
\end{aligned}
\end{align} 
where 
\begin{align*}
	D &\defeq \int_0^1 F_R(q/\omega) f_{\lvert O_i \rvert}(\omega)\d \omega, \\
	I_2 &\defeq \int_0^1 \omega^2 M(q/\omega) f_{\lvert O_i \rvert}(\omega)\d \omega, \\
	I_1 &\defeq \int_0^1 M(q/\omega) f_{\lvert O_i \rvert}(\omega)\d \omega,\\
	M(q/\omega) &\defeq \int_0^{q/\omega} r^2 F_R(\d r).
\end{align*}

Since $\omega \in (0,1)$, the second term in the RHS of \eqref{CondCov3} is nonnegative. On the other hand, since $\omega^2$ is increasing in $\omega$ and $M(q/\omega)$ is decreasing in $\omega$, Lemma \ref{Cheb} proves that $I_2 \le \bfE \| O_i \|_2^2 I_1 = (1/d) I_1$. Therefore, the coefficient of $\lambda_i$ in \eqref{CondCov3} is nonpositive. That is, the preserved total variance decreases with $\lambda_i$. Thus, $\Tr\left(\bfLambda_{i, \alpha}\right)$ is maximized when $\lambda_i= \lambda_ d$ and minimized when $\lambda_i= \lambda_1$.
\end{proof}


The following corollary is a direct consequence of the first equality in \eqref{ConExp}.

\begin{corollary}\lb{CondCovM}
\begin{align}
\begin{aligned}
	\bfLambda_{i,\alpha} &\defeq\Cov\giventhat[\big]{\bfY}{Y_i \in \sfT_{Y_i}^\alpha} \\
	&= \diag(a\lambda_1, \ldots, a\lambda_{i-1}, b\lambda_i, a\lambda_{i+1}, \ldots, a\lambda_d),
\end{aligned}
\end{align}
where $a= (I_1-I_2)/[(d-1)D]$ and $b= I_2/D$.
\end{corollary}

Corollary \ref{CondCovM} also reveals the following important extension to Theorem \ref{Th1}:

\begin{theorem}\lb{GenTh1}
	Under the conditions of Theorem \ref{Th1}, let $\calP$ be the collection of all rank-$d$ orthogonal rotations. For $\bfP \in \calP$, let $\bfP\bfX \defeq \tilde\bfY = (\tilde Y_1, \ldots, \tilde Y_d)$ be the scalar projection of $\bfX$ in the direction of $\bfP$. Then 
	\begin{align}
	\begin{aligned}
		\max_{\bfP \in \calP} \Tr \left(\Cov\giventhat*{\tilde\bfY}{\tilde Y_i \in \sfT_{\tilde Y_i}^\alpha}\right) &= \Tr \left(\Cov\giventhat*{\bfY}{Y_d \in \sfT_{Y_d}^\alpha}\right) \\
		\min_{\bfP \in \calP} \Tr \left(\Cov\giventhat*{\tilde\bfY}{\tilde Y_i \in \sfT_{\tilde Y_i}^\alpha}\right) &= \Tr \left(\Cov\giventhat*{\bfY}{Y_1 \in \sfT_{Y_1}^\alpha}\right),
	\end{aligned}
	\end{align}
	where $\bfY = \bfV' \bfX$ is the scalar projection in the direction of the principal components. 
\end{theorem}

\begin{proof}
	Define $\tilde\bfLambda_{i,\alpha} \defeq \Cov\giventhat*{\tilde\bfY}{\tilde Y_i \in \sfT_{\tilde Y_i}^\alpha}$. Again, by \eqref{ConExp}, 
	\begin{align}\lb{ConCov1}
		\tilde\bfLambda_{i,\alpha} = \diag\left(a\tilde \lambda_1, \ldots, a\tilde\lambda_{i-1}, b\tilde\lambda_i, a\tilde\lambda_{i+1}, \ldots, a\tilde\lambda_d\right),
	\end{align}
	where $a,b$ are defined as in Corollary \ref{CondCovM}. On the other hand, since for all $k \in \{1, \ldots, d\}$,
\begin{align}\lb{OptPC}
	\tilde\lambda_1 + \cdots + \tilde\lambda_k \le \lambda_1 + \cdots + \lambda_k,  
\end{align}	
(see, e.g., Exercise 3.4 of \cite{Vershynin2026}), the result follows. 
\end{proof}

Corollary \ref{CondCovM} also allows us to draw some conclusions about other variance measures and norms of $\tilde\bfLambda_{i,\alpha}$. First, the preserved Frobenius norm behaves like the total variance. Second, the preserved generalized variance remains constant among peeling directions. Third, the preserved operator norm exhibits a behavior between the Frobenius norm and the generalized variance: it is minimized when $i=1$ and constant when $i\ne1$:

\begin{theorem}\lb{Gen}
	Under the conditions of Theorem \ref{Th1}, 
	\begin{enumerate}
		\item The Frobenius norm $\norm{\bfLambda_{i,\alpha}}_F$ is maximized when $i=d$ and minimized when $i=1$.
		\item The generalized variance $\abs{\bfLambda_{i,\alpha}}$ is constant for all $i\in\{1,\ldots,d\}$.
		\item The operator norm $\norm{\bfLambda_{i,\alpha}}_{\mathrm{op}}$ is minimized by peeling the first principal component, and it is constant in all other directions.
	\end{enumerate}
\end{theorem}

\begin{proof}
	{\bf Frobenius norm:}
	\begin{align}\lb{Frob}
	\begin{aligned}
		\norm{\bfLambda_{i,\alpha}}_F^2 &= a^2\sum_{j=1}^d \lambda_j^2 +(b^2 -a^2)\lambda_i^2 \\
							&=a^2\norm{\bfLambda}_F^2 +(b+a)(b-a)\lambda_i^2.
	\end{aligned}
	\end{align}
	The first term in the last equality of \eqref{Frob} is constant on $i$. On the other hand, observe that $a,b \ge 0$, whereas 		$b-a$ equals the coefficient of $\lambda_i$ in the RHS of \eqref{CondCov3}, which is nonpositive. Therefore, the second term is nonpositive, which proves the result.
	
	{\bf Generalized variance:}
	\begin{align*}
		\abs{\bfLambda_{i,\alpha}} = a^{d-1}b\prod_{j=1}^d \lambda_j = a^{d-1}b\abs{\bfLambda},
	\end{align*} 
	and the RHS is independent of $i$. That is, given $\alpha$, the generalized variance is constant for all $i$.
	
	{\bf Operator norm:}
	\begin{align}\lb{Op}
	\begin{aligned}
		\norm{\bfLambda_{i,\alpha}}_{\mathrm{op}} &=\max(a\lambda_1, \ldots, a\lambda_{i-1}, b\lambda_i, a\lambda_{i+1}, \ldots, a\lambda_d)
	\end{aligned}
	\end{align}
	For $i \ne 1$, $\norm{\bfLambda_{i,\alpha}}_{\mathrm{op}} = a\lambda_1 = a\norm{\bfLambda}$, showing that the operator norm is constant by peeling in any direction $i \ne 1$. On the other hand, when $i=1$, $\norm{\bfLambda_{i,\alpha}}_{\mathrm{op}}^2 = \max(b\lambda_1, a\lambda_2) \le a\lambda_1 = a \norm{\bfLambda}$, which minimizes the operator norm.
\end{proof}

Using Theorem \ref{Th1} and Proposition \ref{Gen}, we can prove now Theorem \ref{TheTh}:

\begin{proof}[Proof of Theorem \ref{TheTh}]

	{\bf Preserved total variance:}
	
	Define $\bfLambda_{I_k,\alpha} \defeq \Cov\giventhat*{\bfY}{\bigcap_{j=1}^k \left\{ Y_{i_j} \in \sfT_{Y_{i_j}}^\alpha\right\}}$. Analogously to \eqref{CondCov},
	\begin{align}\lb{CondCovII}
	\begin{aligned}
	\Tr\left(\bfLambda_{I_k,\alpha}\right) &= \bfE\giventhat*{\|\bfY\|^2}{\bigcap_{j=1}^k \left\{ Y_{i_j} \in \sfT_{Y_{i_j}}^\alpha\right\}}\\
		&= \bfE\giventhat*{R^2 \sum_{j=1}^d \lambda_j O_j^2}{\bigcap_{j=1}^k \left\{ \lvert Y_{i_j} \rvert \leq q_{i_j} \right\} } \\
		&= \bfE\giventhat*{R^2 \sum_{j=1}^d \lambda_j O_j^2}{\bigcap_{j=1}^k \left\{ R \lvert O_{i_j} \rvert \leq q_{i_j} \right\}},
\end{aligned}
\end{align} 

Analogously to \eqref{CondCov2},
\begin{align}\lb{CondCovII2}
\begin{aligned}
	&\Tr\left(\bfLambda_{I_k,\alpha}\right) \\
		&= \bigintsss_{[0,1]^k} \bfE\giventhat*{R^2 \sum_{j=1}^d \lambda_j O_j^2}{\bigcap_{j=1}^k \left\{ R \omega_j \leq q \right\}} f_{\lvert \bfO_I \rvert}(\bfomega)\d \bfomega \\
		&= \frac{\int_{[0,1]^k} \bfE\giventhat*{\sum_{j=1}^d \lambda_j O_j^2}{\left\{ \lvert \bfO_I \rvert = \bfomega \right\} }  M(m) f_{\lvert \bfO_I \rvert}(\bfomega)\d \bfomega}{\int_{[0,1]^k} F_R(m) f_{\lvert \bfO_I \rvert}(\bfomega)\d \bfomega},
\end{aligned}
\end{align} 
where $\abs{\bfO_I} \defeq \left(\abs{O_{i_1}}, \ldots, \abs{O_{i_k}}\right)$ is the $k$-dimensional marginal random vector of $\lvert \bfO \rvert \defeq \left( \lvert O_1 \rvert, \ldots, \lvert O_d\rvert \right)$, $\bfomega \defeq (\omega_1, \ldots, \omega_k)$, $M(x) = \int_0^x r^2 F_R(\d r)$, and $m = m(\bfomega) \defeq \min(q/\omega_1, \ldots, q/\omega_k)$. Now, analogously to \eqref{ConExp}, with $k< d$,
\begin{align}\lb{ConExpII} 
\begin{aligned}
	&\bfE\giventhat*{\sum_{j=1}^d \lambda_j O_j^2}{\left\{ \lvert \bfO_I \rvert = \bfomega \right\} } \\
	&= \sum_{j=1}^k \lambda_{i_j} \omega_j^2 + \frac{\Tr(\bfSigma) - \sum_{j=1}^k \lambda_{i_j}}{d-k} \left(1- \sum_{j=1}^k\omega_j^2 \right)\\
	&= \sum_{j=1}^k \lambda_{i_j}\frac{(d - k)\omega_j^2 - 1-  \sum_{j=1}^k \omega_j^2}{d-k} + \Tr(\bfSigma)\frac{1- \sum_{j=1}^k \omega_j^2}{d-k}
\end{aligned}
\end{align}

Define
\begin{align}\lb{Ints}
\begin{aligned}
	D &\defeq \int_{[0,1]^k} F_R(m) f_{\lvert \bfO_I \rvert}(\bfomega)\d \bfomega, \\
	I_2^{(j)} &\defeq \int_{[0,1]^k} \omega_j^2 M(m) f_{\lvert \bfO_I \rvert}(\bfomega)\d \bfomega, \\
	I_1 &\defeq \int_{[0,1]^k} M(m) f_{\lvert \bfO_I \rvert}(\bfomega)\d \bfomega.
\end{aligned}
\end{align}

Since $I_2^{(1)} = \cdots = I_2^{(k)}$ by symmetry, we can simply write it as $I_2$. Therefore, inserting \eqref{ConExpII} into \eqref{CondCovII2},
\begin{align}\lb{CondCovII3}
\begin{aligned}
	\Tr\left(\bfLambda_{I_k,\alpha}\right) 
	&=\sum_{j=1}^k \lambda_{i_j} \frac{d I_2 - I_1}{(d-k)D} + \Tr(\bfSigma)\frac{I_1 - k I_2}{(d-k)D}
\end{aligned}
\end{align} 

Observe that the coefficient of $\Tr(\bfSigma)$ in \eqref{ConExpII} is positive by sphericity, making the coefficient of $\Tr(\bfSigma)$ in \eqref{CondCovII3} positive. On the other hand, for $j \in \{1, \ldots, k \}$, $M$ is decreasing in $m(\omega_1, \ldots, \omega_{j-1}, \cdot, \omega_{j+1}, \ldots, \omega_k)$. Therefore, we can apply Lemma \ref{Cheb} to $I_2$, obtaining that the coefficient of $\lambda_{i_j}$ is negative. The result follows for the space of eigenvectors of $\Sigma$.

When $k=d$, the RHS of \eqref{ConExpII} becomes $\lambda_1\omega_1^2 + \cdots + \lambda_d\omega_d^2$, and the result follows directly.

To extend the result to all orthogonal rotations $\calP$, observe that, for $\tilde\bfLambda_{I_k,\alpha} \defeq \Cov\giventhat*{\tilde\bfY}{\bigcap_{j=1}^k \left\{ \tilde Y_{i_j} \in \sfT_{\tilde Y_{i_j}}^\alpha\right\}}$, by the first equality in \eqref{ConExpII},
\begin{align}
	\tilde\bfLambda_{I_k,\alpha} = \diag \left(b\tilde\lambda_{i_1}, \ldots, b\tilde\lambda_{i_k}, a\tilde\lambda_{i_{k+1}}, \ldots, a\tilde\lambda_{i_d} \right),
\end{align}
where $b = I_2/D$, $a = (I_1 - kI_2)/[(d-k)D]$, and $I_1, I_2, D$ as in \eqref{Ints}. Then, \eqref{OptPC} obtains the result for the space of all rank-$d$ orthogonal transformations.

{\bf Frobenius norm:}

\begin{align}\lb{FrobII}
	\begin{aligned}
		\norm{\bfLambda_{I_k,\alpha}}_F^2 &= a^2\sum_{i=1}^d \tilde\lambda_i^2 +(b^2 -a^2)\sum_{j=1}^k \tilde\lambda_{i_j}^2 \\
							&=a^2\norm{\tilde\bfLambda}_F^2 +(b+a)(b-a)\sum_{j=1}^k \tilde\lambda_{i_j}^2.
	\end{aligned}
	\end{align}
	Reasoning as in \eqref{Frob}, the result is obtained by an application of \eqref{OptPC}.
\end{proof}

\begin{proof}
	{\bf Generalized variance:}
	\begin{align*}
		\abs{\bfLambda_{I_k,\alpha}} = a^{d-k}b&k\prod_{j=1}^d \tilde\lambda_j = a^{d-k}b^k\abs{\tilde\bfLambda} = a^{d-k}b^k\abs{\bfLambda},
	\end{align*} 
	where the last equality is obtained because, for any $\bfP \in \calP$, $\abs{\bfP\Sigma\bfP^T} = \abs{\Sigma}$. That is, given $\alpha$, the preserved generalized variance is constant for all rotations.
	
	{\bf Operator norm:}
	\begin{align}\lb{Op}
	\begin{aligned}
		\norm{\tilde\bfLambda_{I_k,\alpha}}_{\mathrm{op}} &=\max\left(b\tilde\lambda_{i_1}, \ldots, b\tilde\lambda_{i_k}, a\tilde\lambda_{i_{k+1}}, \ldots, a\tilde\lambda_{i_d}\right)
	\end{aligned}
	\end{align}
	For $i \notin I_k$, $\norm{\tilde\bfLambda_{I_k,\alpha}}_{\mathrm{op}} = a\tilde\lambda_1 = a\norm{\tilde\bfLambda}_{\mathrm{op}}$, showing that the operator norm is constant by peeling in any direction $i \ne 1$. On the other hand, when $i \in I_k$, $\norm{\tilde \bfLambda_{I_k,\alpha}}_{\mathrm{op}}^2 = \max(b\tilde\lambda_1, a\tilde\lambda_2) \le a\tilde\lambda_1 = a \norm{\tilde\bfLambda}_{\mathrm{op}}$, which minimizes the operator norm. Since $a\tilde\lambda_1 \le a\lambda_1$ by \eqref{OptPC}, the proof is complete.
\end{proof}

\begin{proof}[Proof of Proposition \ref{Th2}]
By \eqref{Ti}, $\Vol \left( \sfT_{I_k, \tilde Y} \right) = \prod_{j=1}^k 2q_{i_j} = (2q)^k \prod_{j=1}^k \tilde \lambda_{i_j}^{k/2}$, which is maximized and minimized when the $k$ principal components of largest and smallest variance, respectively, are chosen.
\end{proof}

\section{Continuation of Fashion MNIST example}\lb{CntEx}

This Appendix extends the Fashion MNIST example from Section \ref{Example} by adding additional categories.

\textbf{Shirts} (Fig.\@~\ref{shirt}): the principal subset captures classic button-down shirts in plain or subtly patterned fabrics, while the pettiest subset reveals ruffled blouses, off-shoulder tops, bold prints, and elaborate decorative styles spanning a far wider range of garment types. \textbf{Dresses} (Fig.\@~\ref{dress}): the principal subset covers a broad representative cross-section of dress lengths and silhouettes; the pettiest subset leans toward more unusual items---mini dresses, asymmetric cuts, and styles that blur the boundary between dress and top. \textbf{Coats} (Fig.\@~\ref{coat}): the principal subset spans casual to formal mid- and full-length coats, while the pettiest subset reveals more distinctive designs---lightweight puffer jackets, camouflage patterns, and technical or quilted styles. \textbf{Sandals} (Fig.\@~\ref{sandal}): the principal subset groups everyday flat or low-wedge sandals, while the pettiest subset is dominated by stiletto heels, strappy evening sandals, and platform wedges---an entirely different footwear character. \textbf{Ankle boots} (Fig.\@~\ref{ankle}): the principal subset concentrates on classic low-heeled Chelsea-style boots in dark tones, while the pettiest subset uncovers high-heeled, platform, metallic, and exotic-patterned designs. \textbf{Bags} (Fig.\@~\ref{bag}): the principal subset selects compact cross-body bags, clutches, and structured shoulder bags, while the pettiest subset spans large tote bags, sports bags, backpacks, and luxury handbags---a much wider range of shapes and purposes. \textbf{Sneakers} (Fig.\@~\ref{sneaker}): the principal subset concentrates on classic low-profile trainers, while the pettiest subset includes high-tops, platform trainers, and chunky-soled styles. \textbf{Trousers} (Fig.\@~\ref{trouser}): the principal subset covers classic full-length trousers and jeans across a range of tones, while the pettiest subset clearly diverges---including shorts, joggers, sweatpants, and patterned leggings.

\begin{figure}[t]
\centering
{\includegraphics[width=4cm,height=3.6cm]{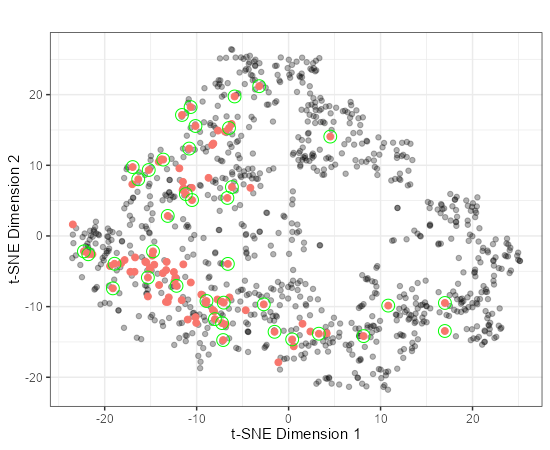}}\quad\raisebox{0.0\height}
{\includegraphics[width=4cm,height=3.6cm]{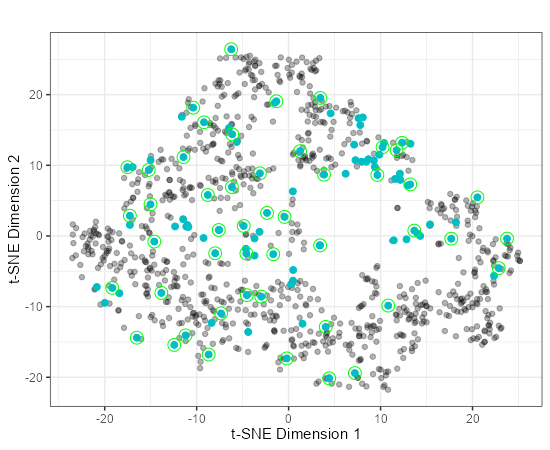}}
\\[1em]
{\includegraphics[width=4cm,height=4.4cm]{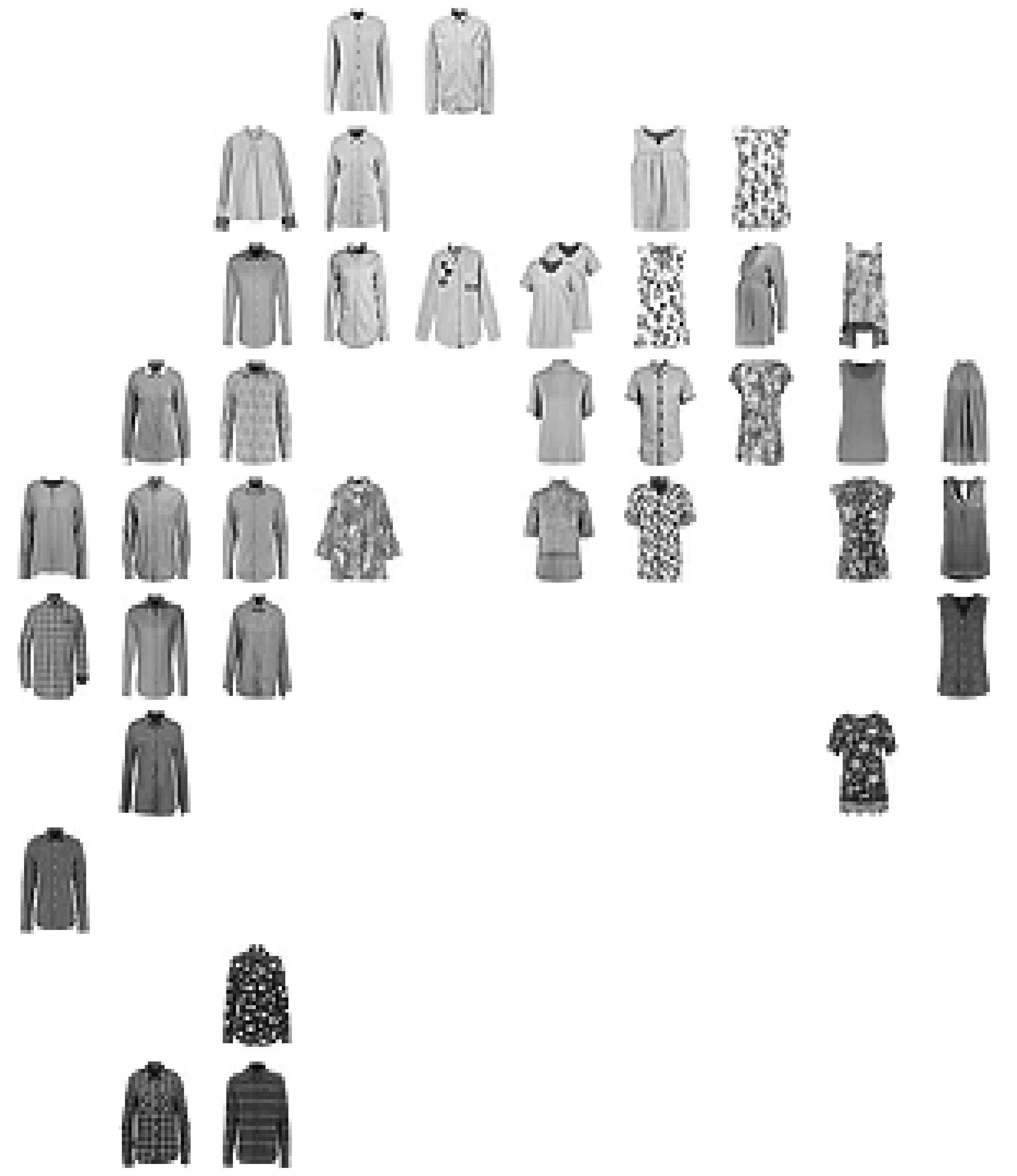}}\quad\raisebox{0.0\height}
{\includegraphics[width=4cm,height=4.4cm]{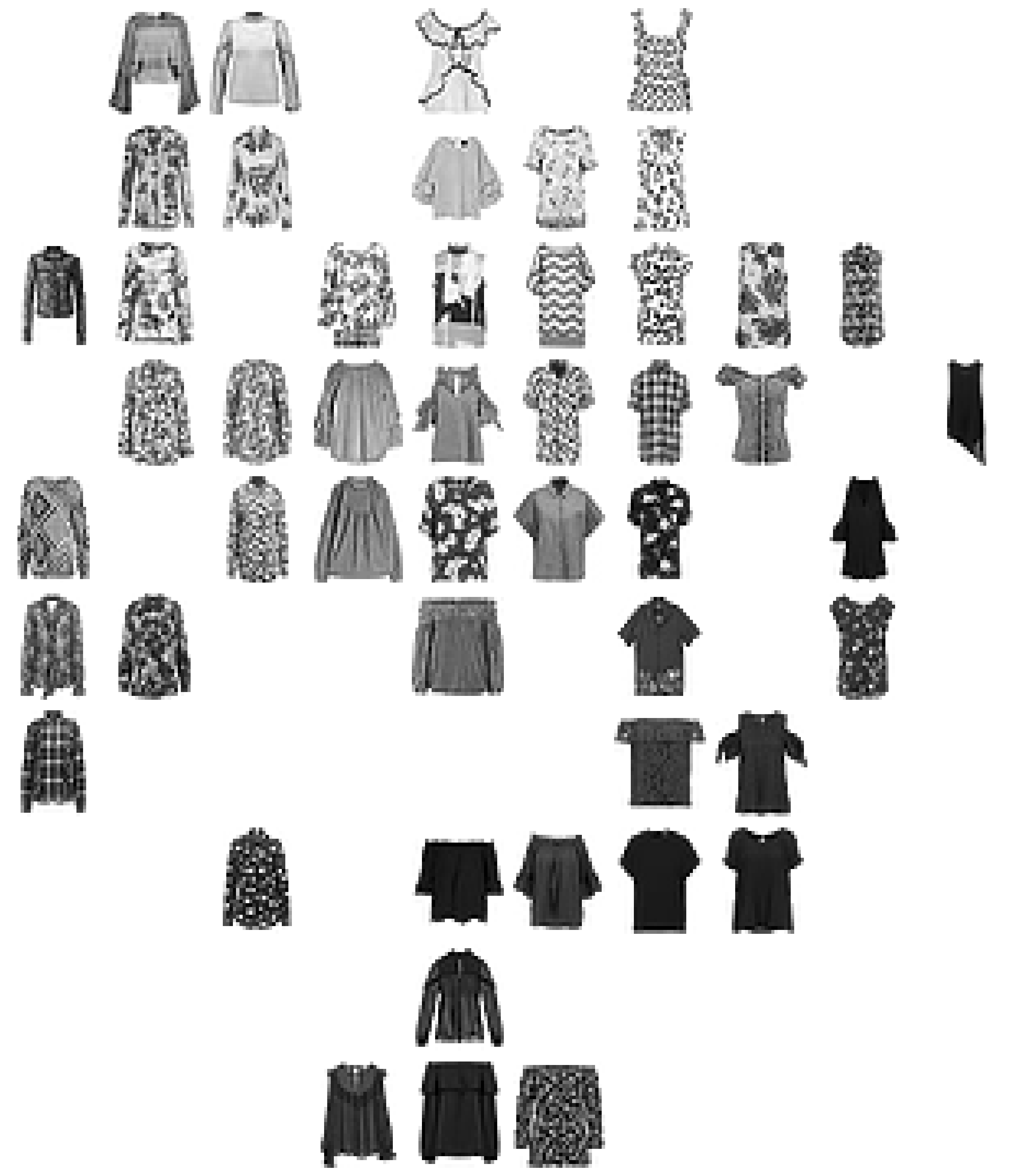}}
  \caption{Shirt example. Left) peeling principal directions. Right) peeling pettiest
  directions} \label{shirt}
\end{figure}

\begin{figure}[t]
\centering
{\includegraphics[width=4cm,height=3.6cm]{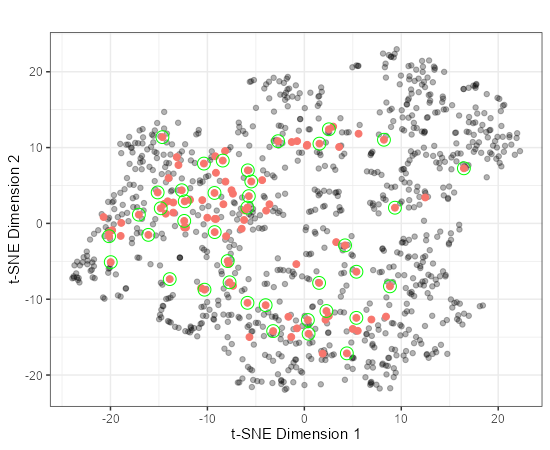}}\quad\raisebox{0.0\height}
{\includegraphics[width=4cm,height=3.6cm]{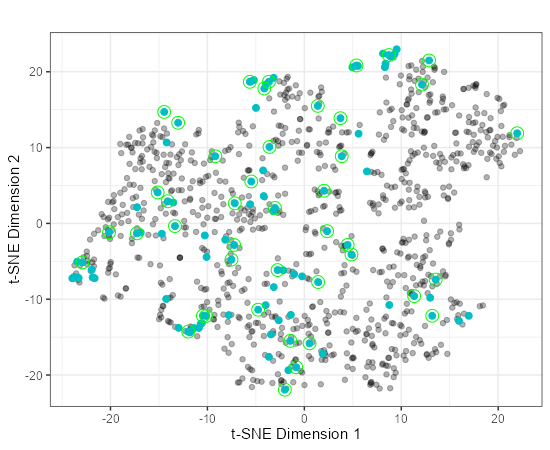}}
\\[1em]
{\includegraphics[width=4cm,height=4.4cm]{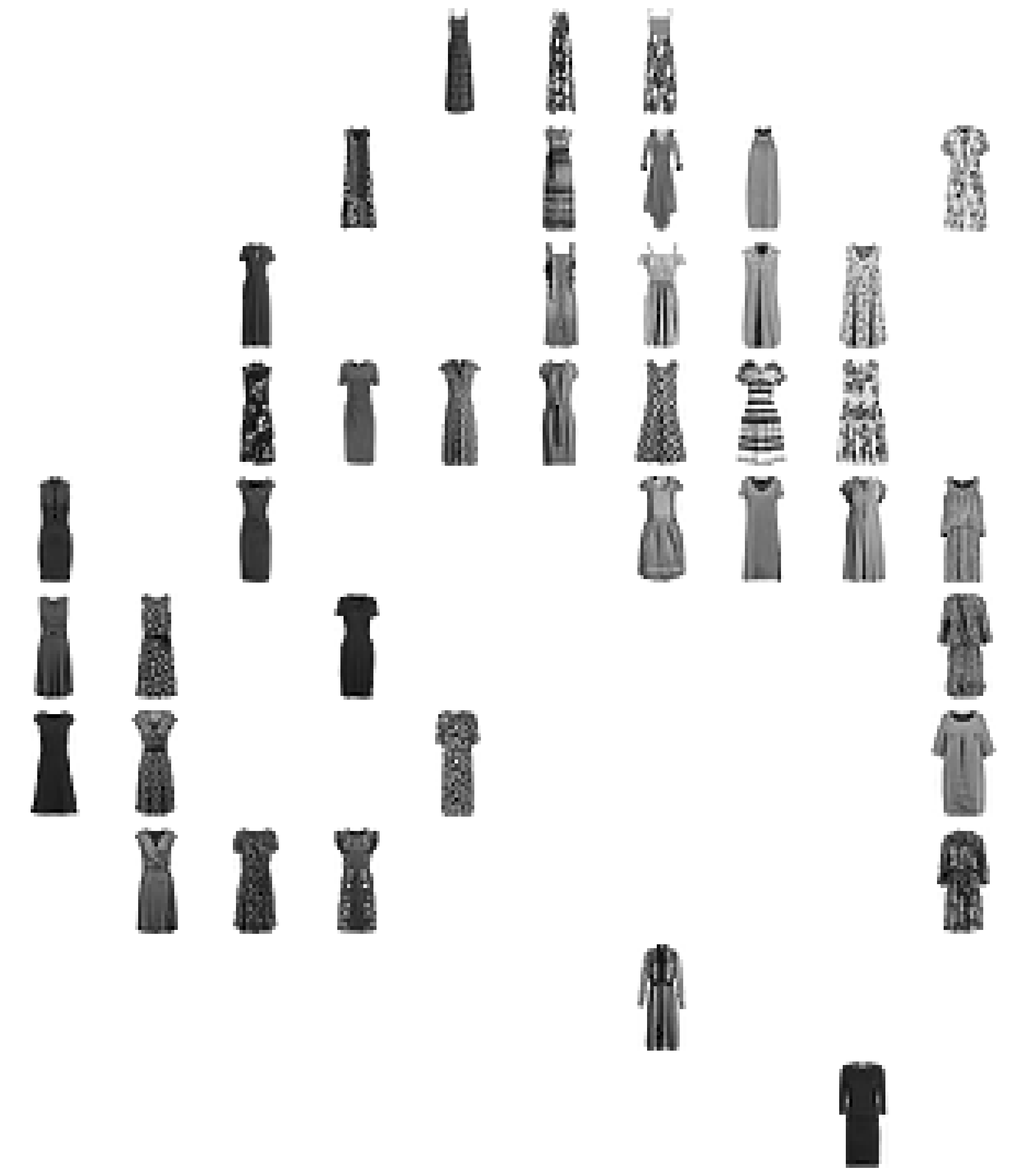}}\quad\raisebox{0.0\height}
{\includegraphics[width=4cm,height=4.4cm]{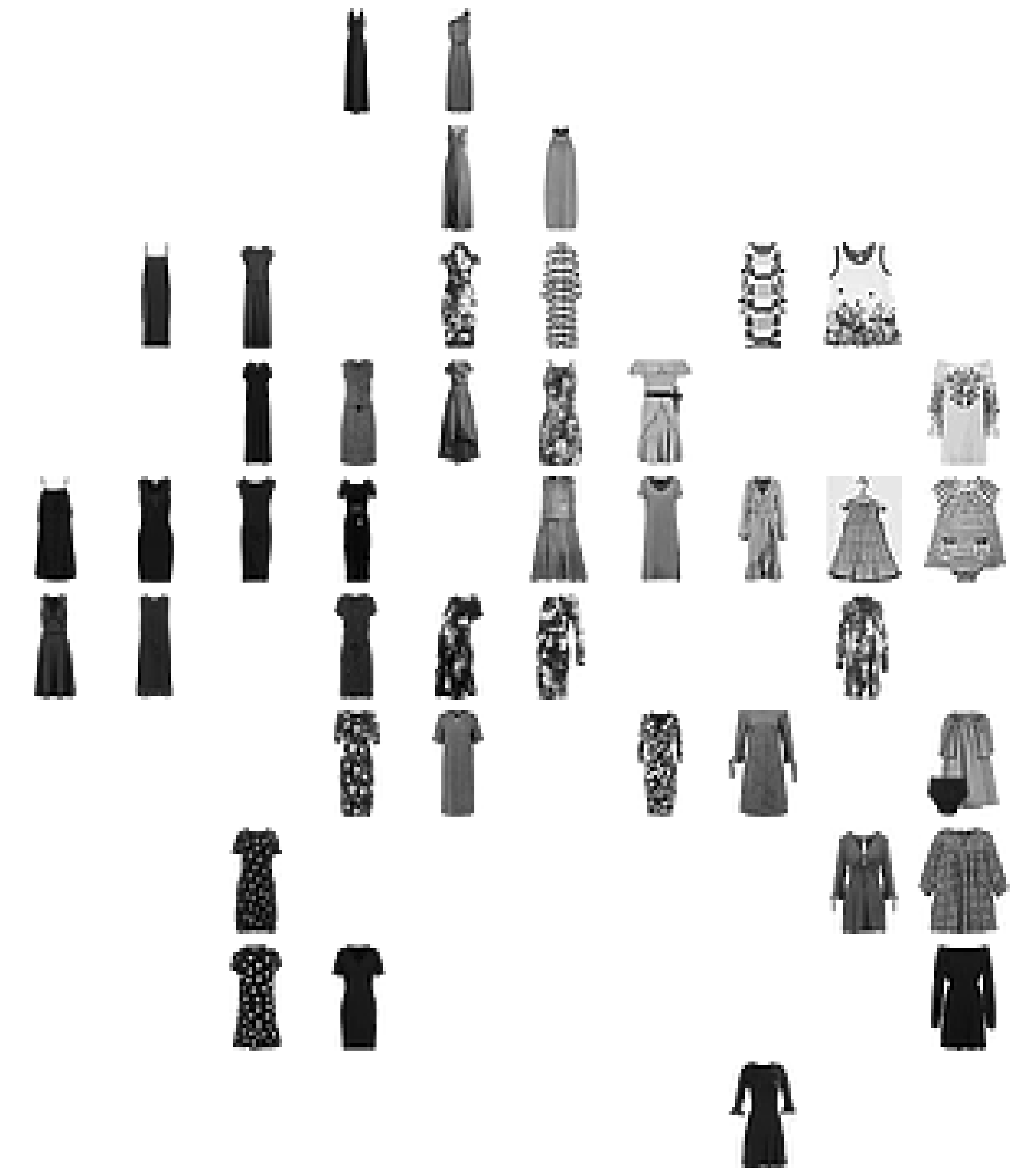}}
  \caption{Dress example. Left) peeling principal directions. Right) peeling pettiest
  directions} \label{dress}
\end{figure}

\begin{figure}[t]
\centering
{\includegraphics[width=4cm,height=3.6cm]{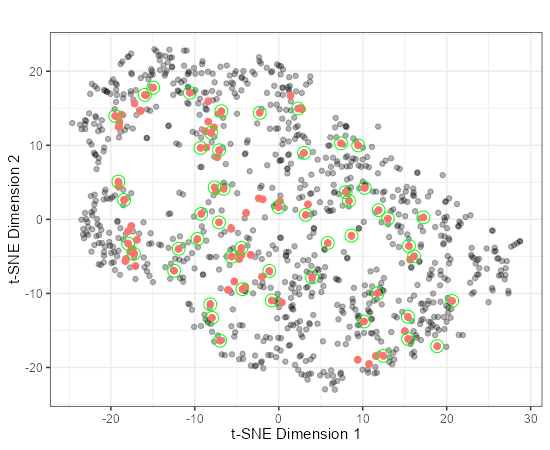}}\quad\raisebox{0.0\height}
{\includegraphics[width=4cm,height=3.6cm]{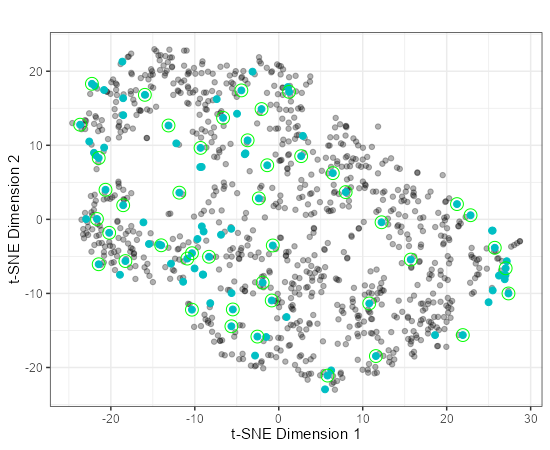}}
\\[1em]
{\includegraphics[width=4cm,height=4.4cm]{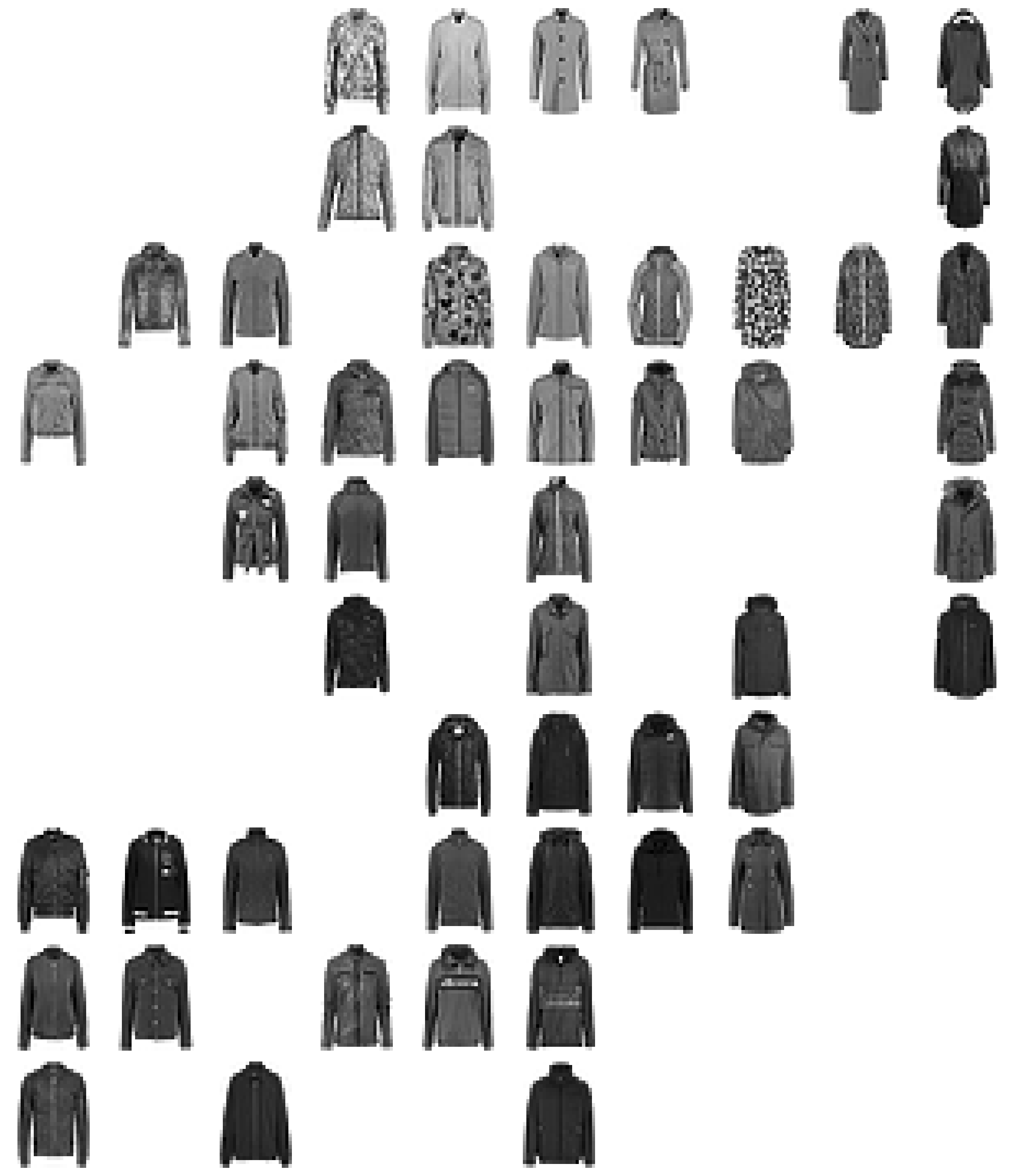}}\quad\raisebox{0.0\height}
{\includegraphics[width=4cm,height=4.4cm]{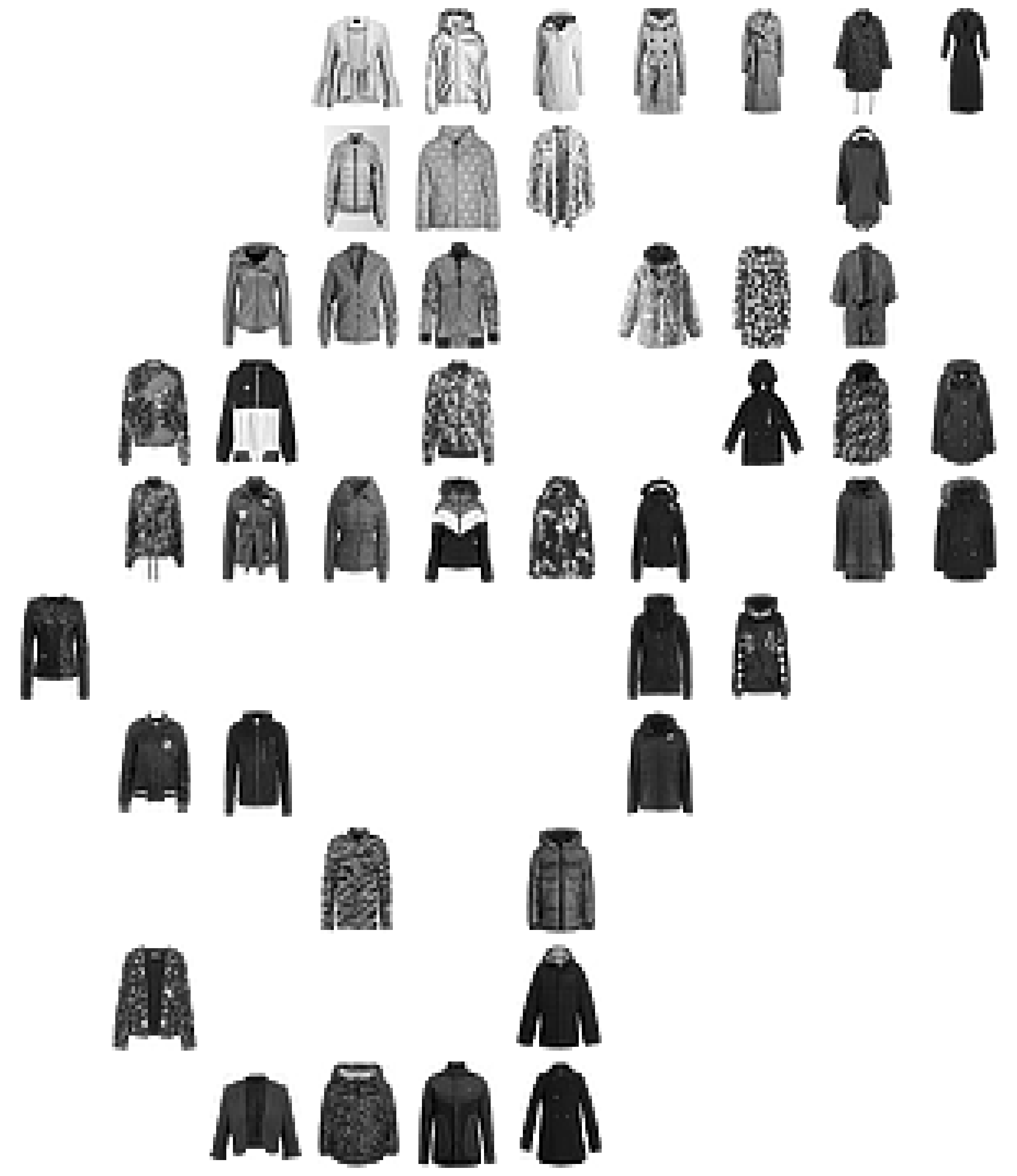}}
  \caption{Coat example. Left) peeling principal directions. Right) peeling pettiest
  directions} \label{coat}
\end{figure}

\begin{figure}[t]
\centering
{\includegraphics[width=4cm,height=3.6cm]{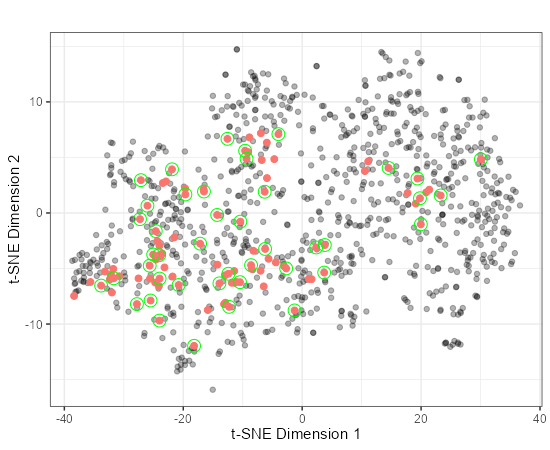}}\quad\raisebox{0.0\height}
{\includegraphics[width=4cm,height=3.6cm]{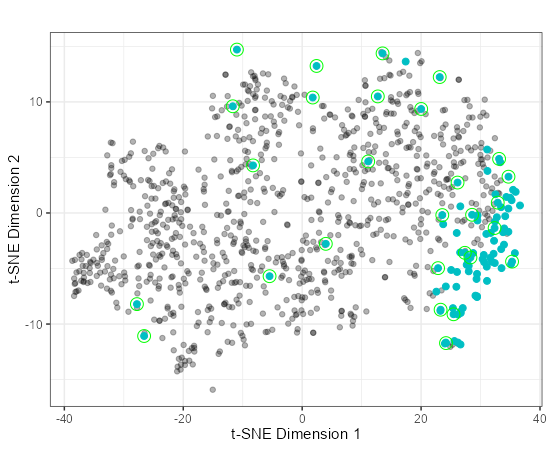}}
\\[1em]
{\includegraphics[width=4cm,height=4.4cm]{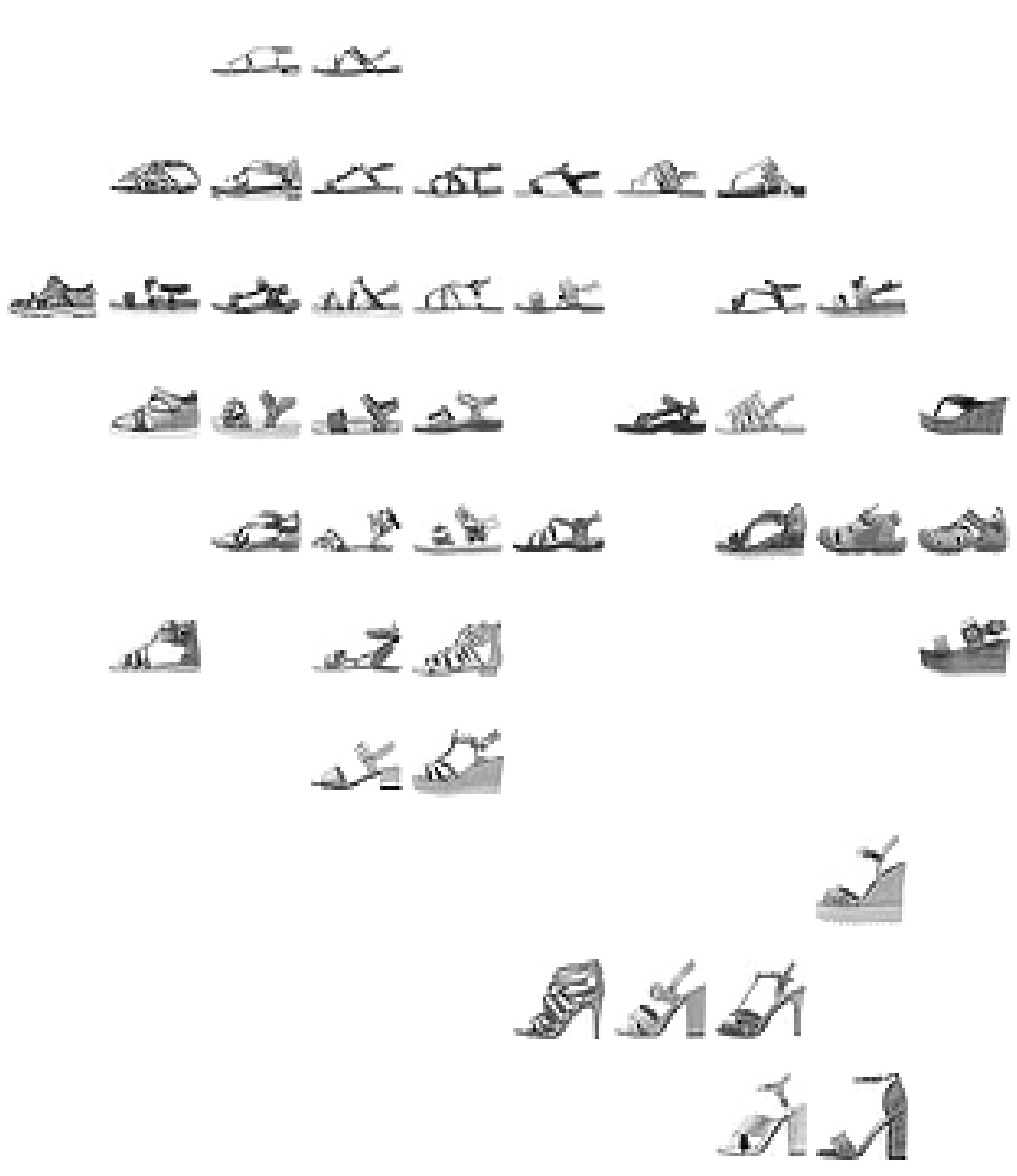}}\quad\raisebox{0.0\height}
{\includegraphics[width=4cm,height=4.4cm]{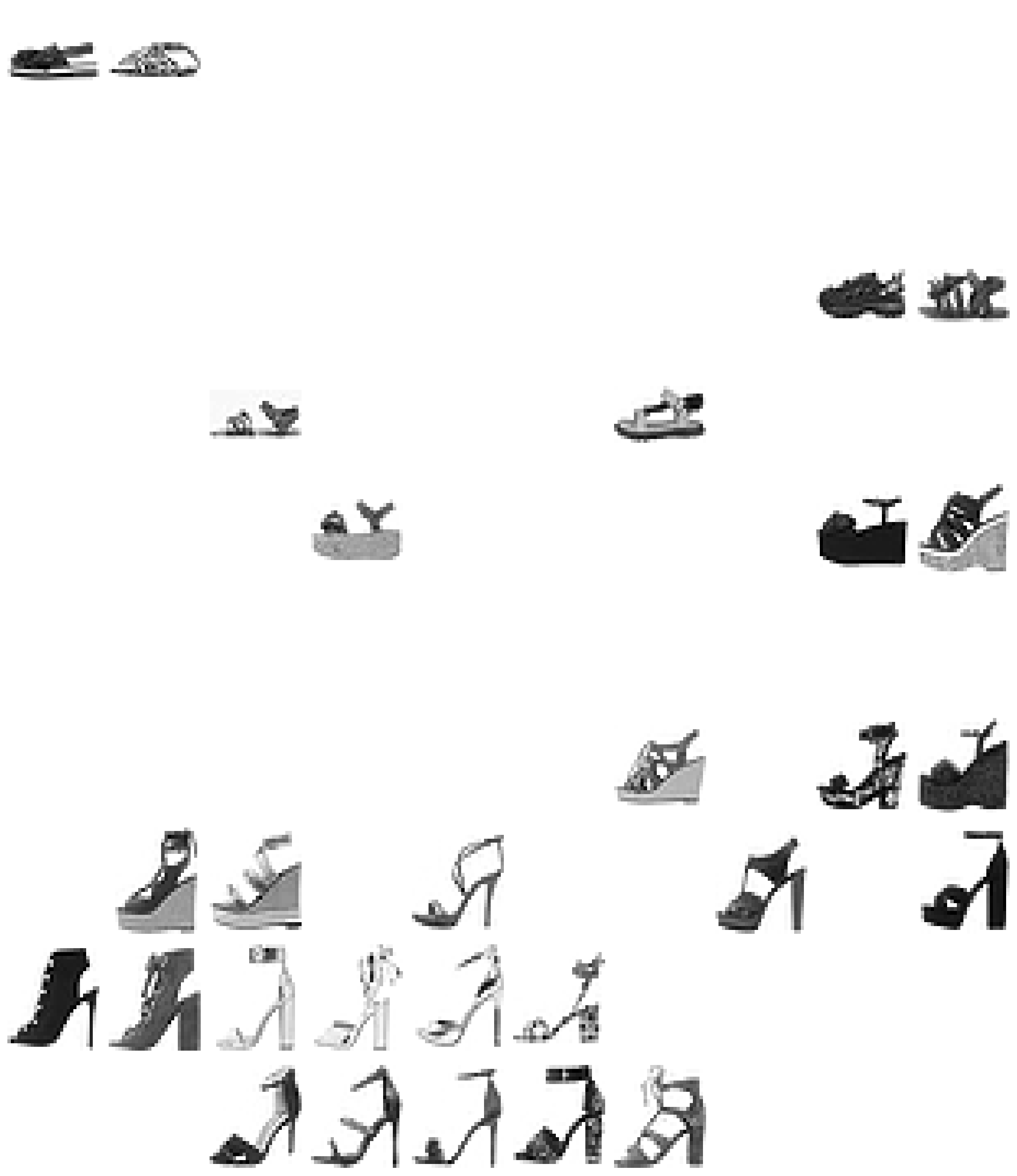}}
  \caption{Sandal example. Left) peeling principal directions. Right) peeling pettiest
  directions} \label{sandal}
\end{figure}

\begin{figure}[t]
\centering
{\includegraphics[width=4cm,height=3.6cm]{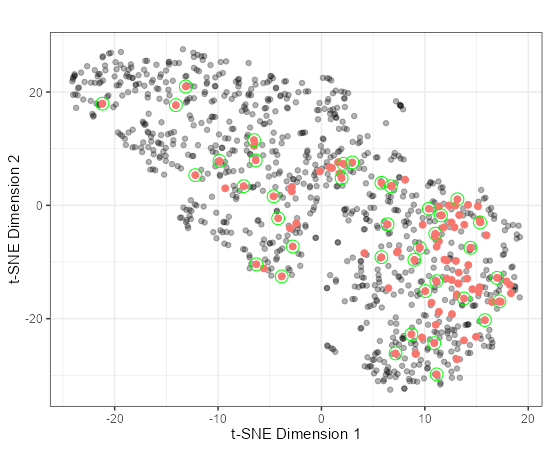}}\quad\raisebox{0.0\height}
{\includegraphics[width=4cm,height=3.6cm]{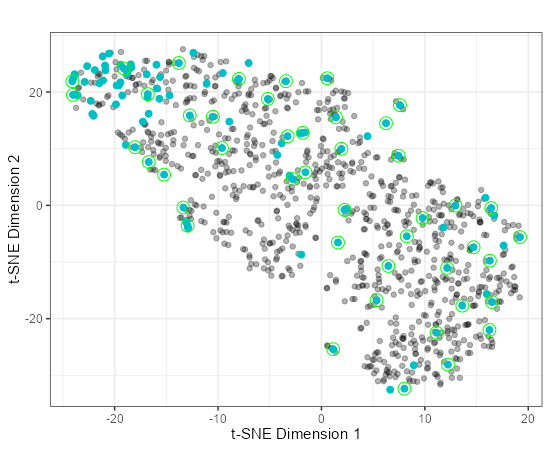}}
\\[1em]
{\includegraphics[width=4cm,height=4.4cm]{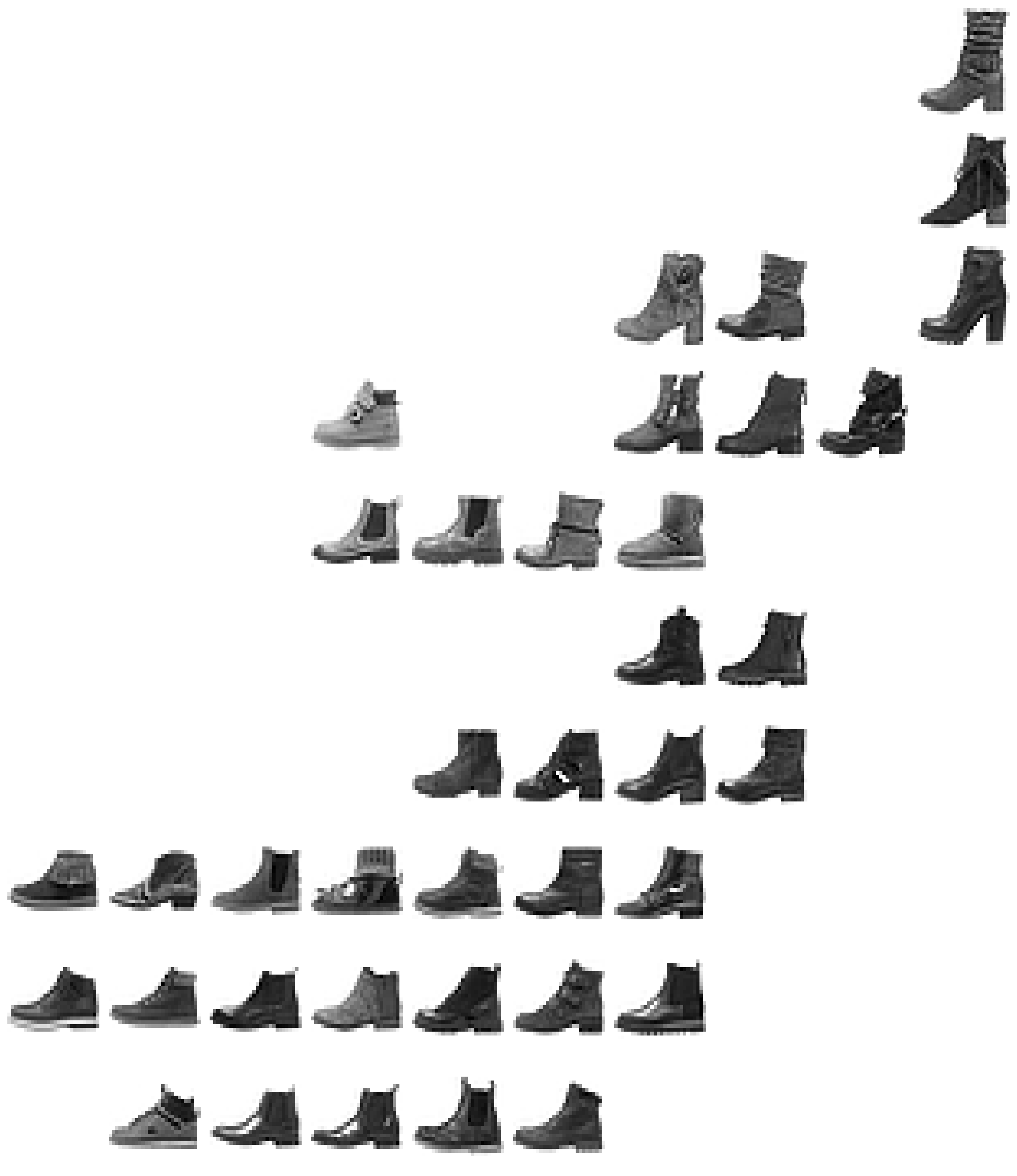}}\quad\raisebox{0.0\height}
{\includegraphics[width=4cm,height=4.4cm]{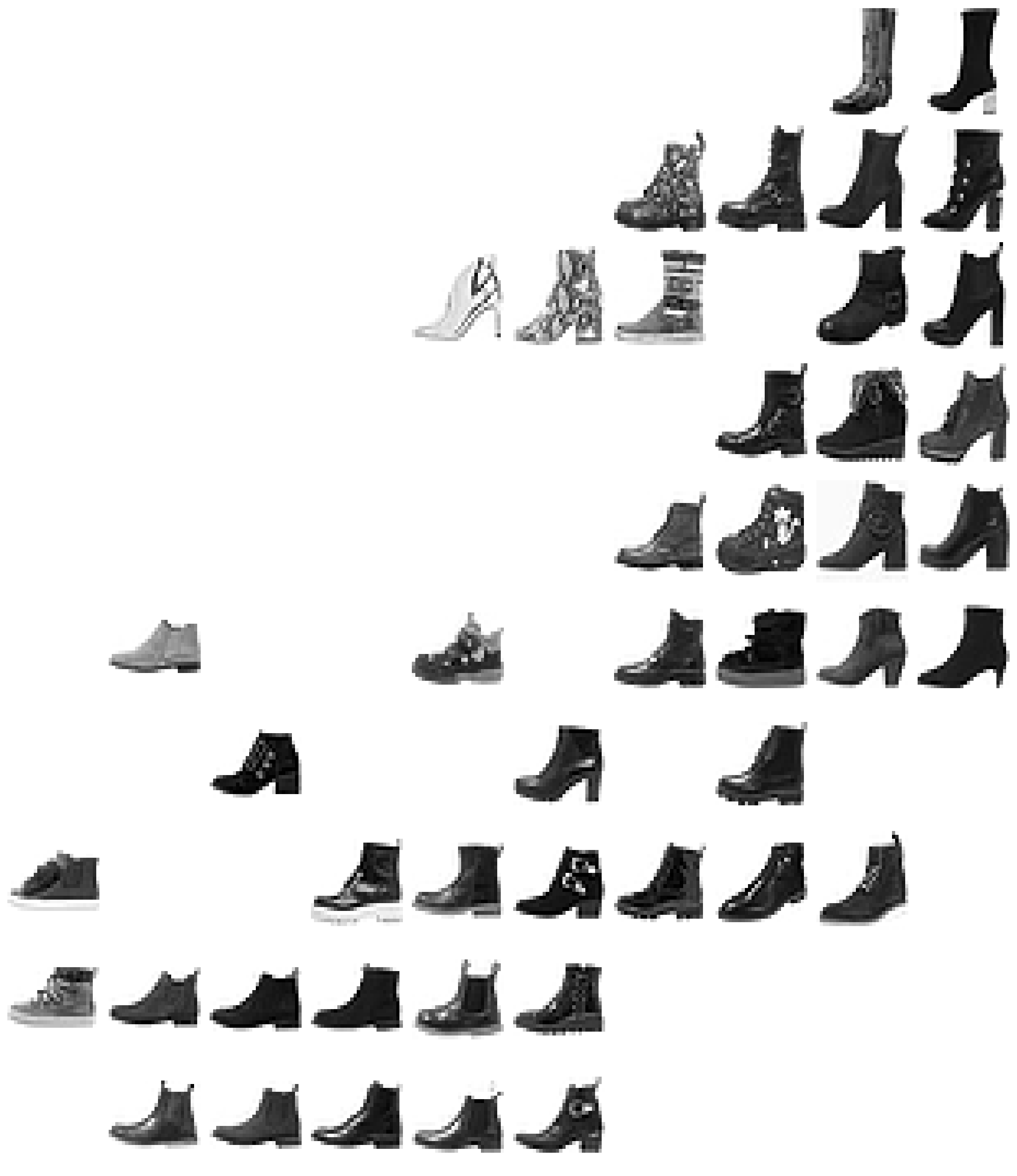}}
  \caption{Ankle boot example. Left) peeling principal directions. Right) peeling pettiest
  directions} \label{ankle}
\end{figure}

\begin{figure}[t]
\centering
{\includegraphics[width=4cm,height=3.6cm]{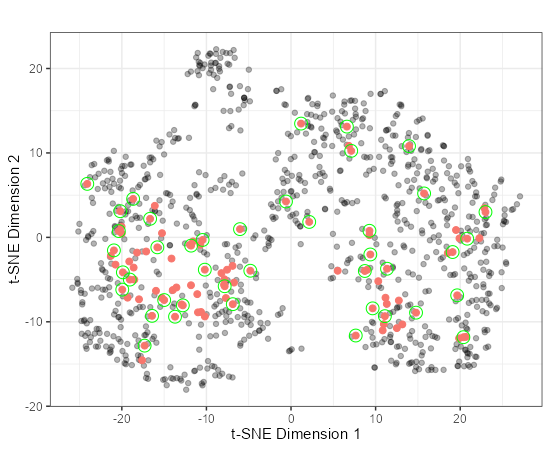}}\quad\raisebox{0.0\height}
{\includegraphics[width=4cm,height=3.6cm]{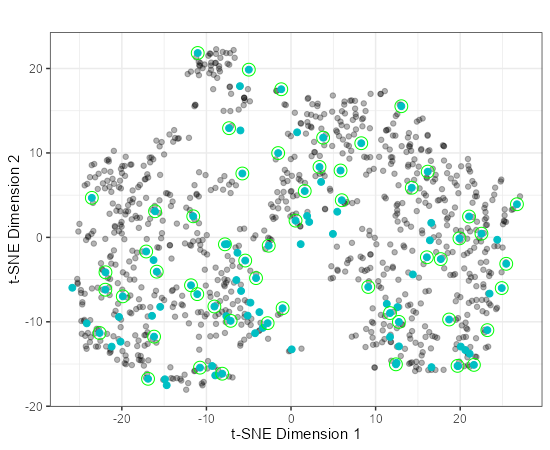}}
\\[1em]
{\includegraphics[width=4cm,height=4.4cm]{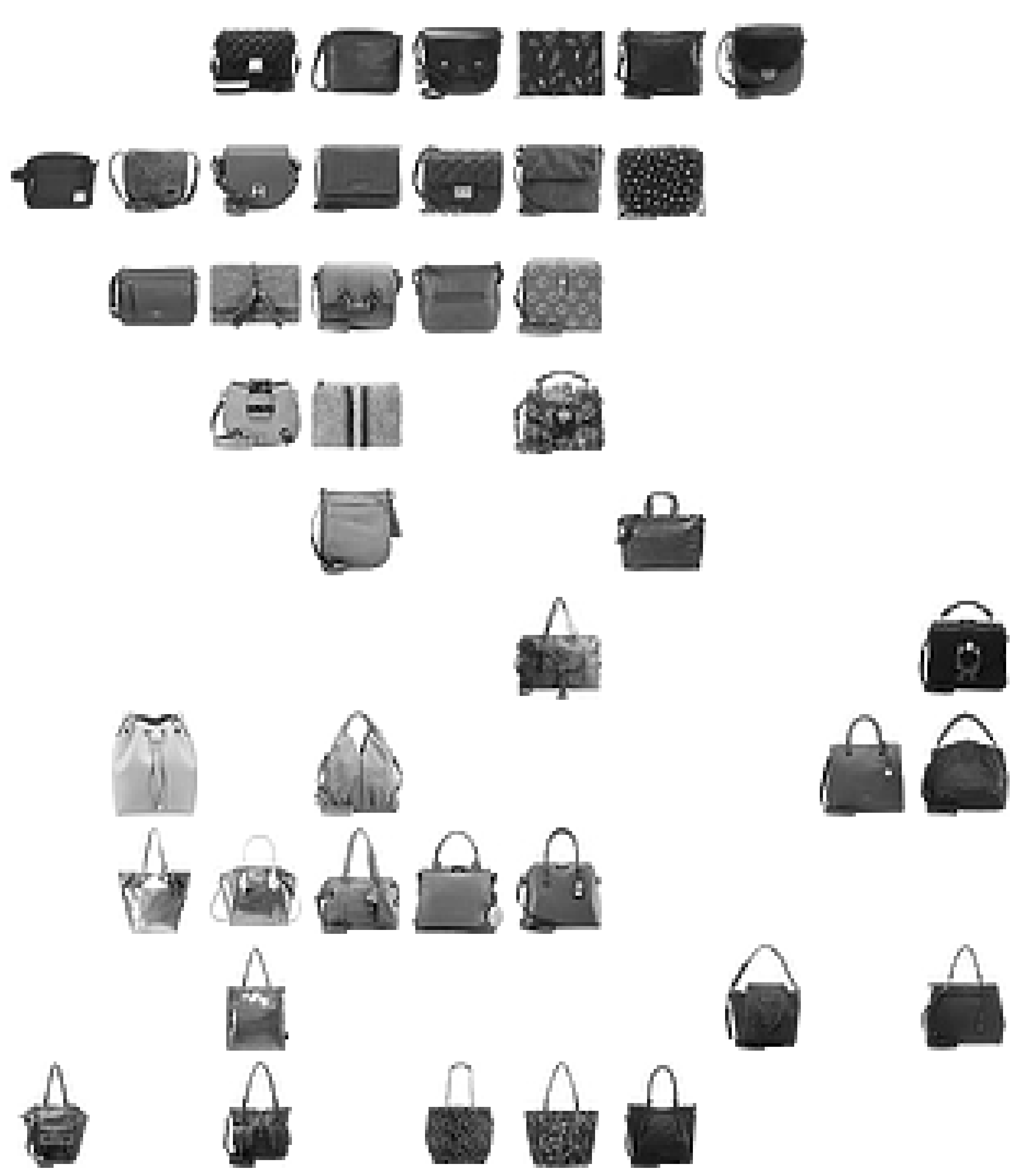}}\quad\raisebox{0.0\height}
{\includegraphics[width=4cm,height=4.4cm]{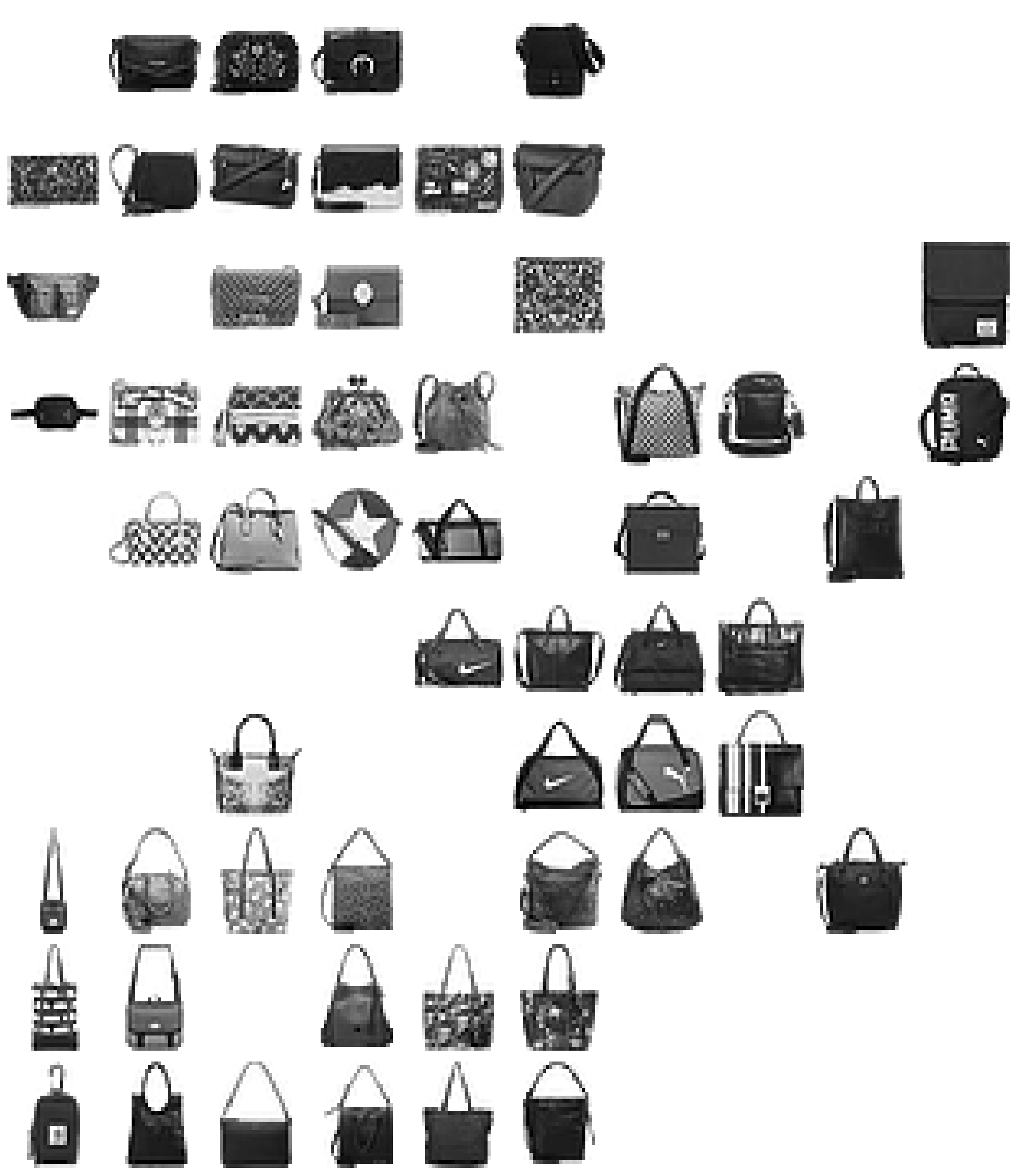}}
  \caption{Bag example. Left) peeling principal directions. Right) peeling pettiest
  directions} \label{bag}
\end{figure}

\begin{figure}[t]
\centering
{\includegraphics[width=4cm,height=3.6cm]{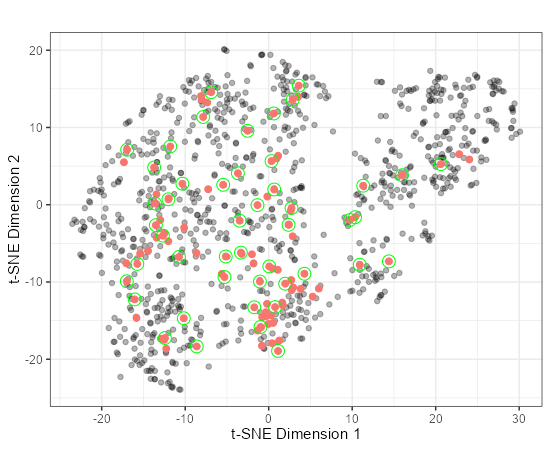}}\quad\raisebox{0.0\height}
{\includegraphics[width=4cm,height=3.6cm]{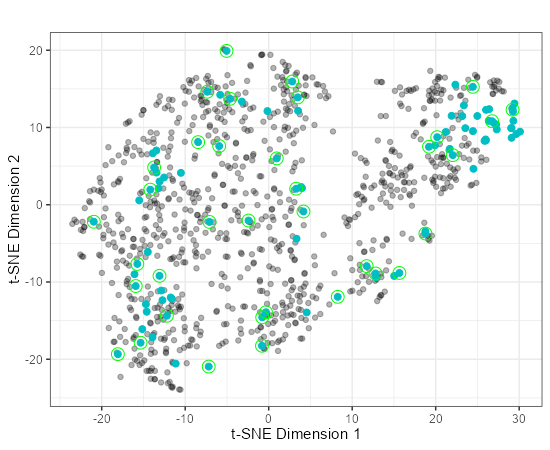}}
\\[1em]
{\includegraphics[width=4cm,height=4.4cm]{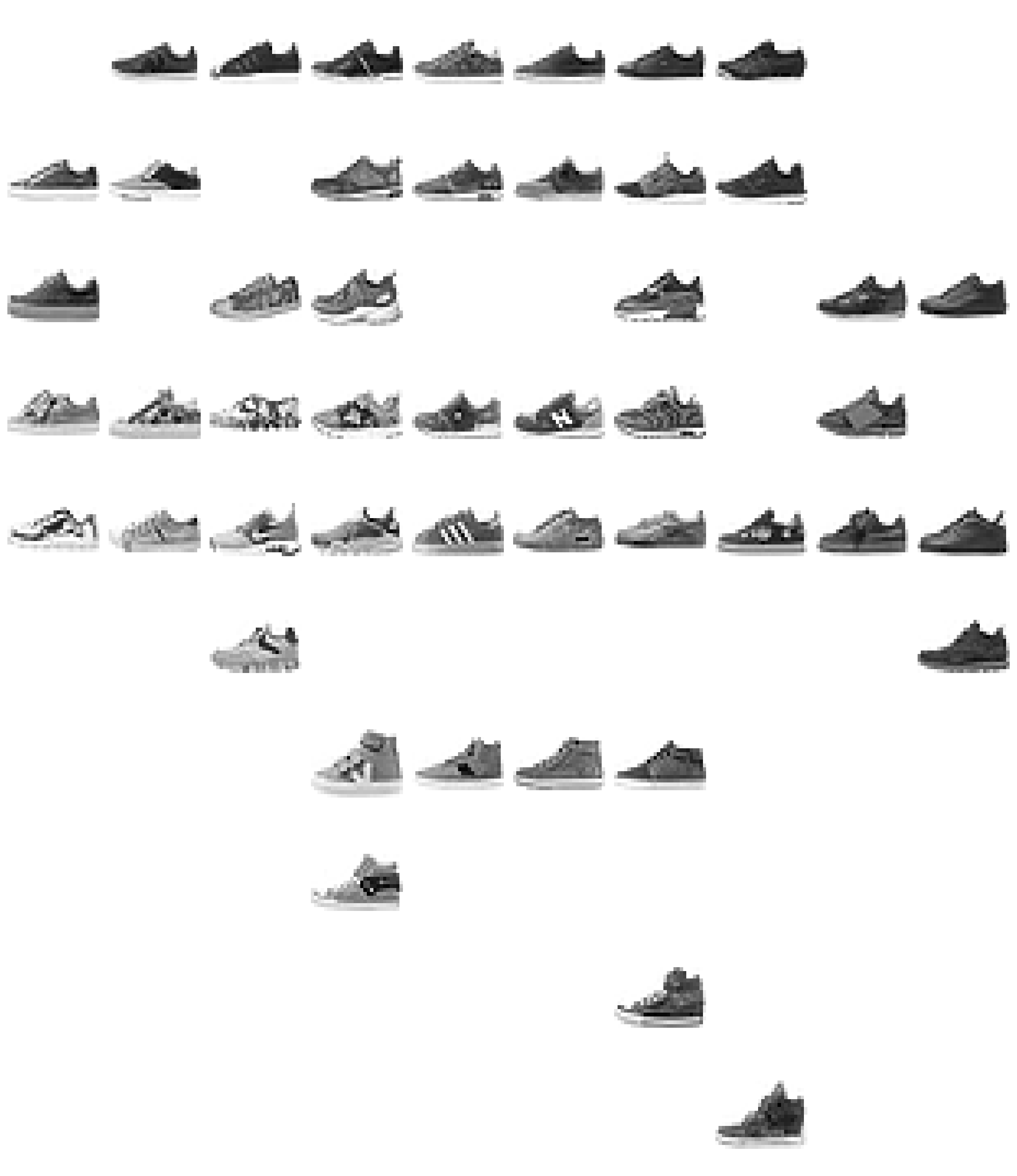}}\quad\raisebox{0.0\height}
{\includegraphics[width=4cm,height=4.4cm]{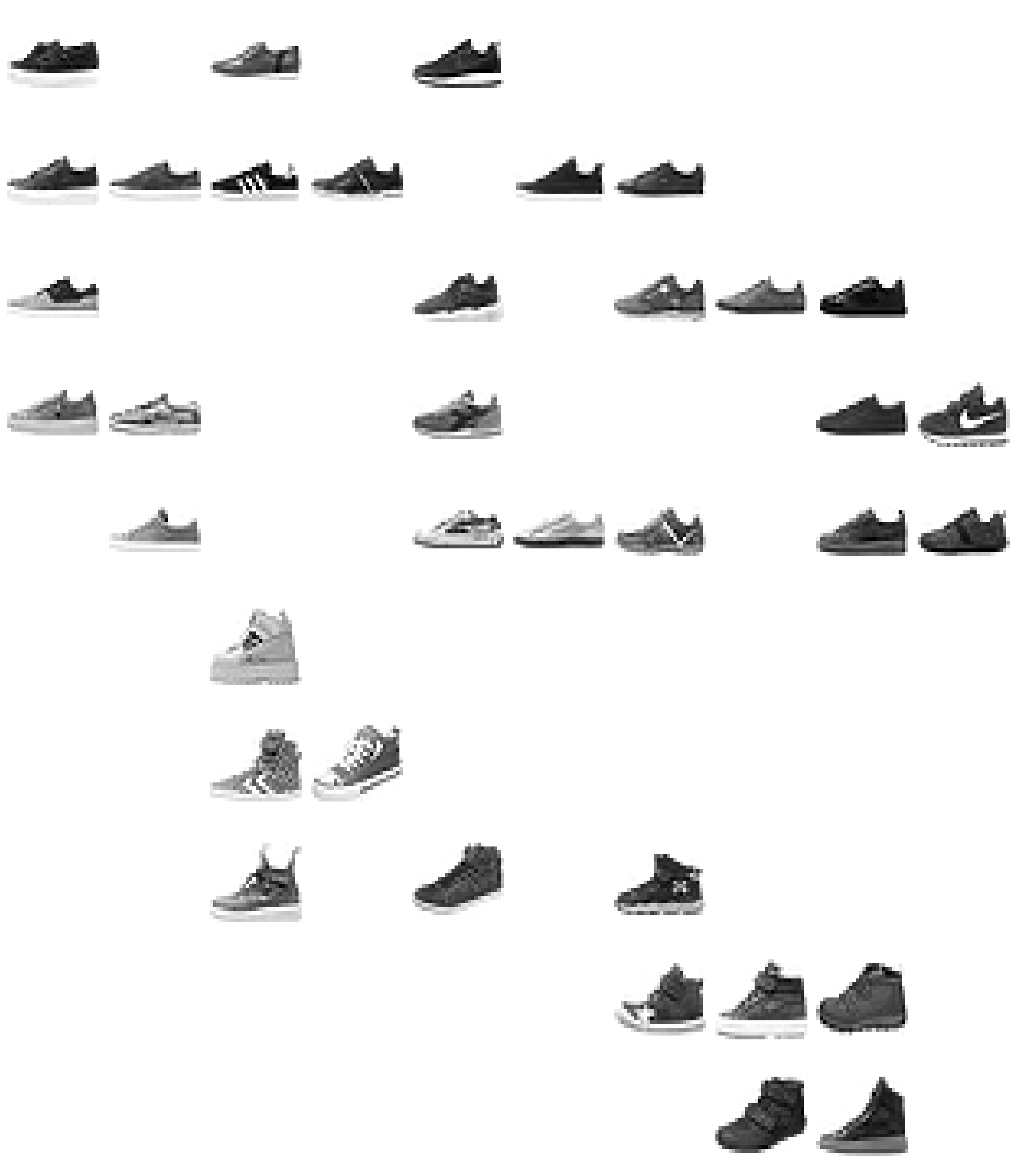}}
  \caption{Sneaker example. Left) peeling principal directions. Right) peeling pettiest
  directions} \label{sneaker}
\end{figure}

\begin{figure}[t]
\centering
{\includegraphics[width=4cm,height=3.6cm]{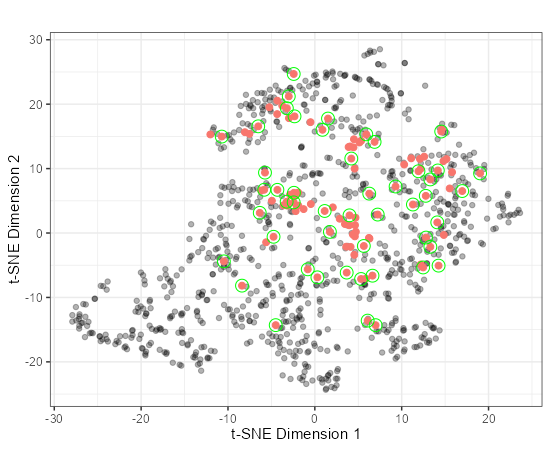}}\quad\raisebox{0.0\height}
{\includegraphics[width=4cm,height=3.6cm]{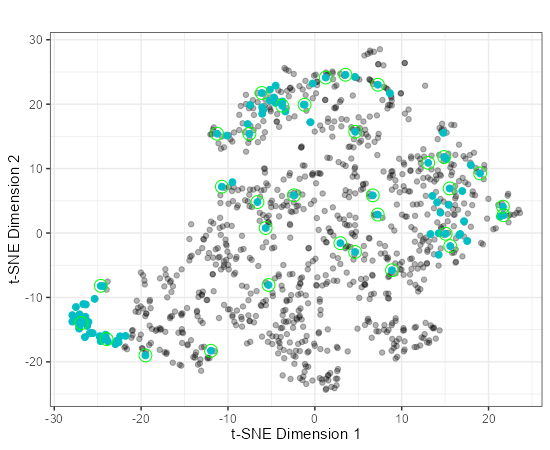}}
\\[1em]
{\includegraphics[width=4cm,height=4.4cm]{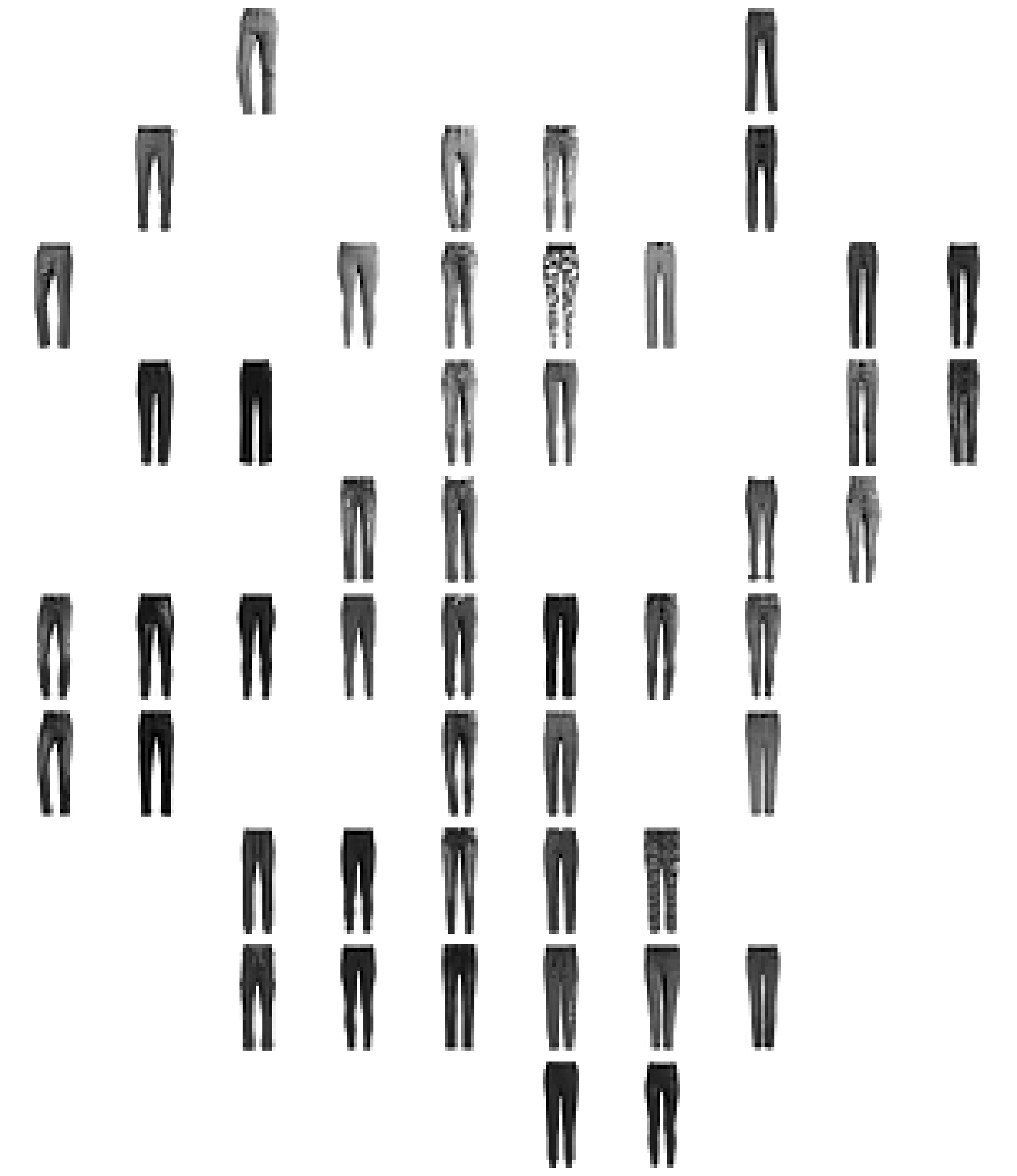}}\quad\raisebox{0.0\height}
{\includegraphics[width=4cm,height=4.4cm]{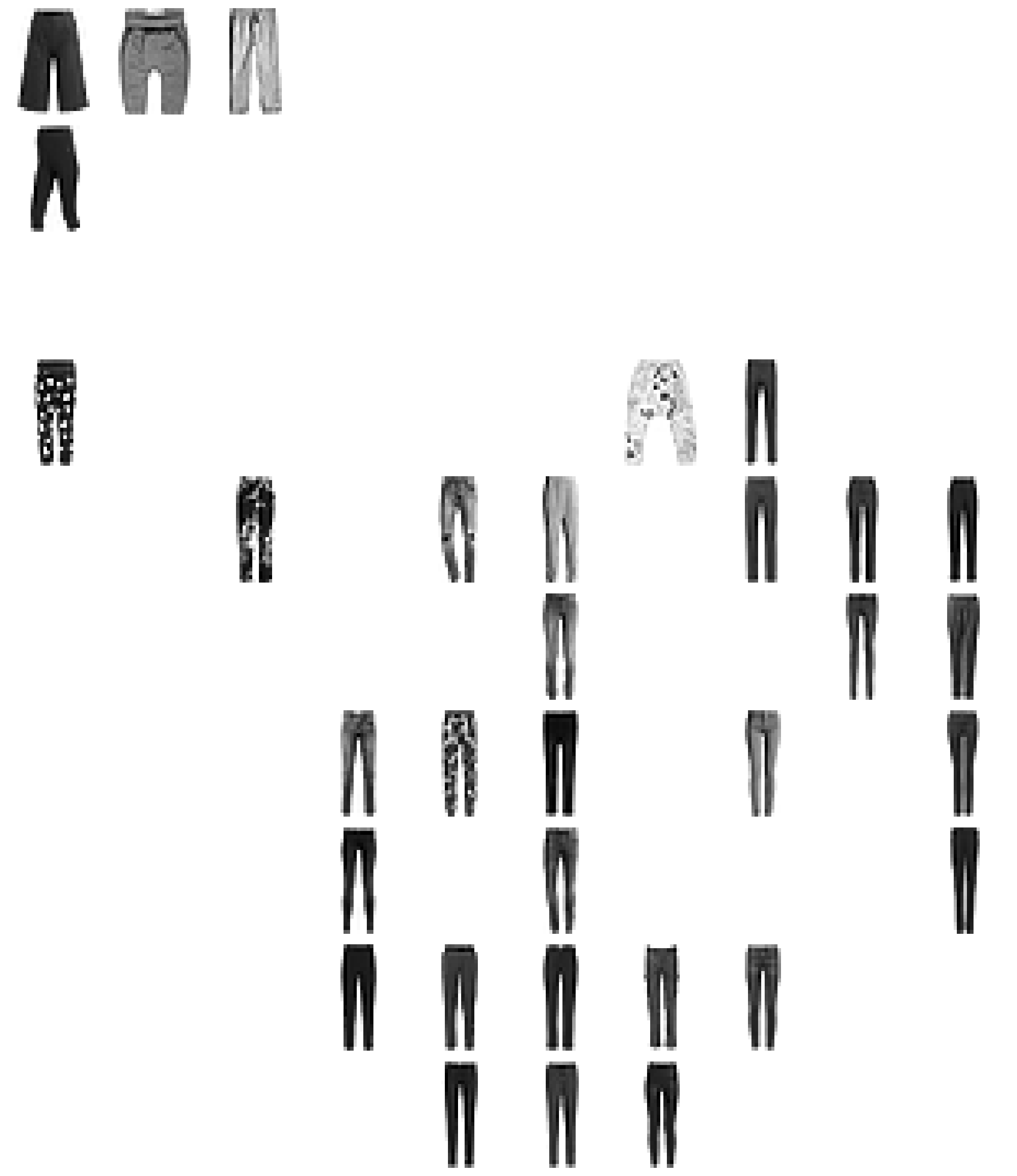}}
  \caption{Trouser example. Left) peeling principal directions. Right) peeling pettiest
  directions} \label{trouser}
\end{figure}

\section*{Acknowledgment}

The authors thank Grok (xAI) for suggesting the stochastic representation $\bfX = R\bfA\bfO$ and the monotonicity argument via Lemma \ref{Cheb}, which formed the basis of the one-dimensional proof of Theorem \ref{Th1}. This suggestion was rigorously validated, extended to the full multivariate elliptical case, and completed by the authors.

\ifCLASSOPTIONcaptionsoff
  \newpage
\fi



\bibliographystyle{IEEEtran}
\bibliography{/Users/daangapa/Documents/Research/daangapaBibliography.bib}

\begin{IEEEbiographynophoto}{Tianhao Liu}
received his BS in Physics from Nankai University, China, in 2019, and a MS in Biostatistics from University of Miami, in 2020. He is currently working towards a PhD in Biostatistics at the University of Miami. His research interests are mainly in statistical learning, information theory, and their applications to medicine.
\end{IEEEbiographynophoto}

\begin{IEEEbiographynophoto}{Daniel Andrés Díaz--Pachón}
is a Research Assistant Professor with the Division of Biostatistics at the University of Miami. His research focuses on probability, information theory, mathematical statistics, and theoretical machine learning. He has applied these tools to cosmology, the origin of life, population genetics, and infectious diseases. He enjoys spending time with his family, working out, reading, and writing.
\end{IEEEbiographynophoto}

\begin{IEEEbiographynophoto}{J. Sunil Rao}
is a Professor with the Division of Biostatistics, University of Minnesota, Twin Cities, where he is also the Director of Biostatistics with the Masonic Cancer Center. His research interests include mixed-model prediction and selection, Bayesian model selection, small-area estimation, machine learning, and applied biostatistics, with a focus on cancer and health disparities.
\end{IEEEbiographynophoto}

\end{document}